%% file: rccmab_arxiv.tex
\documentclass[12pt, draftclsnofoot, onecolumn]{IEEEtran}

\usepackage{ifpdf}
\usepackage{cite}
\usepackage[pdftex]{graphicx}

\usepackage[T1]{fontenc}

\usepackage{afterpage}
\usepackage{bm}
\usepackage{listings}
\usepackage{xcolor}
\usepackage{graphicx}
\usepackage{float}
\usepackage[font=small]{caption}
\usepackage[font=footnotesize]{subcaption}

\usepackage{placeins}
\usepackage{url}
\usepackage{epstopdf}
\usepackage{booktabs}
\usepackage{footnote}
\usepackage{tikz}
\makesavenoteenv{table}
\makesavenoteenv{tabular}
\usepackage{flushend}
\usepackage{amsmath}
\usepackage{amsthm}
\usepackage{amssymb}
\usepackage{tabularx}
\usepackage{booktabs}
\usepackage{placeins}
\usepackage{graphicx}
\usepackage{color}
\usepackage{colortbl}
\usepackage{bbold}
\DeclareMathOperator*{\argmax}{arg\,max}
\usepackage[acronyms,nonumberlist,nopostdot,nomain,nogroupskip]{glossaries}
\usepackage{tabularx}
\usepackage{paralist}
\usepackage{algorithm}
\usepackage{algorithmic}

\newsavebox{\bigimage}

\makeatother

\newtheorem{theorem}{Theorem}[section]
\newtheorem{lemma}[theorem]{Lemma}
\input{math_commands.tex}

\usepackage{glossaries}
\input{acronyms.tex}
\IEEEoverridecommandlockouts
\newcommand\copyrightnotice{%
	\begin{tikzpicture}[remember picture,overlay]
	\node[anchor=north,yshift=-15pt] at (current page.north) {\parbox{\dimexpr\textwidth-\fboxsep-\fboxrule\relax}{
			\centering\footnotesize This paper has been submitted to the IEEE JSAIT's special issue on Modern Compression. Copyright may change without notice.}};
	\end{tikzpicture}
}

\begin{document}

\title{Rate-Constrained Remote Contextual Bandits}

\author{
\IEEEauthorblockN{\large Francesco Pase$^{*\dagger}$, Deniz G\"und\"uz$^\dagger$, Michele Zorzi$^*$ \vspace{1mm}}
    
    \IEEEauthorblockA{
		\small $^*$Department of Information Engineering, University of Padova, Italy. Email: \{pasefrance, zorzi\}@dei.unipd.it
	}
	
	\IEEEauthorblockA{
		\small $^\dagger$ Department of Electrical Engineering, Imperial College London, UK. Email: d.gunduz@imperial.ac.uk
	}
}
\maketitle
\copyrightnotice
\begin{abstract}
We consider a rate-constrained contextual multi-armed bandit (RC-CMAB) problem, in which a group of agents are solving the same \gls{cmab} problem. However, the contexts are observed by a remotely connected entity, i.e., the decision-maker, that updates the policy to maximize the returned rewards, and communicates the arms to be sampled by the agents to a controller over a rate-limited communications channel. This framework can be applied to personalized ad placement, whenever the content owner observes the website visitors, and hence has the context, but needs to transmit the ads to be shown to a controller that is in charge of placing the marketing content. Consequently, the \gls{rccmab} problem requires the study of lossy compression schemes for the policy to be employed whenever the constraint on the channel rate does not allow the uncompressed transmission of the decision-maker's intentions. We  characterize the fundamental information theoretic limits of this problem by letting the number of agents go to infinity, and study the regret that can be achieved, identifying the two distinct rate regions leading to linear and sub-linear regrets respectively. We then analyze the optimal compression scheme achievable in the limit with infinite agents, when using the forward and reverse KL divergence as distortion metric. Based on this, we also propose a practical coding scheme, and provide numerical results.
\end{abstract}

\glsresetall


\IEEEpeerreviewmaketitle

\section{Introduction}
\label{sec:intro}

The past decade has seen a transition from centralized computing solutions, based on the cloud, to more distributed systems, also known as multi-access edge computing \cite{mahadev2017}, mainly in response to the emerging paradigm of Internet of everything, where data is generated, processed and consumed by a network of connected nodes with diverse storage and computing capabilities. Indeed, today massive amounts of data are generated by edge devices, e.g., sensors, cameras, vehicles and drones, and need to be quickly processed to satisfy the latency constraints of the new vertical services, e.g., \gls{iiot} systems, autonomous driving scenarios or remote surgery. Moving data from distributed and personal devices to the cloud introduces several issues related to privacy, network resources (i.e., offloading data collected by edge devices to central entities may congest the communication networks) and latency.

Another paradigm shift involves the adoption of \gls{ml} as a core technology for future services. \gls{ml} techniques, mainly driven by deep \glspl{nn}, achieve  state-of-the-art performance in many practical applications such as computer vision, speech recognition and automatic control \cite{Chen:PIEE:2019}. However, these impressive performance results require extremely large models, which, in turn, need a lot of data to be trained. This demands the careful design of the distributed learning systems \cite{chen2021distributed, Park2019_edge}, that are foreseen to produce and process most of the data in future communications networks. Consequently, there is an urgent need for new methods to represent and compress the information generated and processed by distributed learning solutions to efficiently make use of the available physical network resources. To this end, many recent works have tried ro reduce the amount of data needed to properly train a distributed model at the edge. One potential solution is \gls{nn} compression, which reduces the amount of information needed to store or transmit an \gls{nn}.

\glsreset{cmab}
The focus of this work is to study the information-theoretic limits of distributed learning in the specific context of multi-agent \glspl{cmab}.
\glspl{cmab} model decision-making problems in which an agent interacts with an environment in sequential rounds. At each round, the agent observes a context, which contains side information on the environment, and has to pull one out of $K$ arms. Based on the observed context and pulled arm, the environment returns a reward, which is sampled according to an unknown distribution. The goal of the agent is to optimize an arm selection strategy to maximize the average sum of obtained rewards. In this scenario, the contexts of $N$ agents are available to a central \textit{decision-maker} that has to inform a remote entity, called the \textit{controller}, on the arms the agents should pull. Depending on the scenario, the controller either directly pulls the arms, or provides the indices of the suggested arms to the agents, that can physically interact with the environment. However, we assume that there is a rate-limited communication channel between the decision-maker and the controller, which limits the number of bits the decision-maker can use to convey the intended arms at each round \footnote{Part of the results presented in this paper have been submitted to the 2022 International Symposium on Information Theory.}

In this paper, we study the problem of communicating policies from the decision-maker to the controller when the available rate is below what would be required to perfectly convey the decision-maker's policy at each iteration, and provide theoretical and practical analysis of lossy policy compression schemes, that take into account the learning objective when specifying the distortion metric to be used. We then draw connections between our framework and the \textit{information bottleneck} \cite{tishby_2000, goyal2018transfer} and \textit{maximum entropy reinforcement learning} \cite{Levine2018, Haarnoja2018} approaches. Finally, we report numerical results in support of our analyses.

\section{Related Work}
\label{sec:related_work}

Future communication networks should be designed with the aim of efficiently supporting edge \gls{ai} applications \cite{chen2021distributed}. In particular, the data exchanged in such systems is correlated with some training process, and so new data compression and representation methods should be considered. The limits of lossy compression have been defined by rate-distortion theory \cite{cover:IT}, which provides the minimum number of bits needed to represent information given a target value for the maximum tolerable distortion between the original data, and the data reconstructed from the compressed version. However, in distributed \gls{ml}, the goal at the receiver is not to perfectly reconstruct the input data, but to perform some inference or training task based on it. In the past literature, these problems have been formulated as distributed hypothesis testing \cite{Berger_1979, Ahlswede-Csiszar} and estimation \cite{Zhang:NIPS:13, Xu:IT:17} with rate-constrained communication links. 

In particular, this work focuses on \gls{marl} \cite{Busoniu:tsmc:2008}, where the goal is not to learn a function through supervised learning, but rather to optimize a policy, i.e.,  a map that, given an observation, provides a probability distribution over a set of feasible actions. Whenever agents take actions, they receive stochastic rewards distributed according to some unknown distributions. In general, the goal of the agents is to learn a policy that maximizes the returned rewards. In our case, the processes of collecting observations, training the policy and taking the actions are implemented by physically distributed entities \cite{Busoniu:tsmc:2008}, which require the design of compression and communication strategies to coordinate the training process. Another application of \gls{marl} is the parallel training of a single logical agent, with the aim of accelerating the training process \cite{pmlr-v48-mniha16, babaeizadeh2016ga3c, clemente:arxiv:2017}.

To this end, an important line of research is the study of the most valuable data to be shared among the agents in order to achieve cooperation, and thus convergence to better policies. Frameworks in \cite{foerster_learning_2016, sukhbaatar_learning_2016, havrylov_emergence_2017, lazaridou_multi-agent_2017} admit cooperation among the agents through the transmission of signals, whose effects on the common task are differentiable with respect to the signal transmission decisions. In this case, the emergence of languages and cooperative messages is analyzed, with perfect communications links. The authors in \cite{Tung_rl_gunduz} consider noisy communications, and propose a scheme to jointly learn good policies and optimal ways to code and transmit the actions over the communications channels.

Other research efforts consider deep \gls{rl} agents, and focus on compressions schemes of \glspl{nn}, which are generally used to approximate value functions and target policies. For example, the authors in \cite{zhang:ijcnn:2019} adopt a knowledge distillation technique to reduce the complexity of a behavioral policy, while training a more complex target policy that serves as a teacher, in the context of parallel training \cite{clemente:arxiv:2017}. In \cite{livne:jstsp:2020} and \cite{Arnob:nips:2021}, the authors achieve model compression adopting pruning techniques \cite{blalock2020, berivan2021, Han:iclr:2016} to reduce the size of an agent's \gls{nn}-based policy. Both papers provide empirical studies of the performance of a single agent when trained with different pruning levels. In \cite{livne:jstsp:2020}, an iterative process is used to prune the model, which is reduced after training the original (full size) \gls{nn}. In our framework, training is performed over rate-limited communications channels, and there could be no opportunity to transmit the whole policy. The authors in \cite{Arnob:nips:2021} consider performing pruning exploiting an initial offline dataset, that is not present in our case. Moreover the pruning level, and so the needed rate to transmit the models, has to be fixed at the beginning of the training process.

As mentioned above, the focus of our analyses is on a multi-agent version of \gls{cmab}, in which a decision-maker observes the contexts for a set of distributed agents, which can interact  with the decision-maker through a remotely connected controller. The framework is similar to that of \cite{fragouli}, where a server communicates actions to a pool of agents. However, in this case the analyzed link is the one between the agents and the server, and is used to send back the observed rewards. Moreover, no context (or state) is considered. In \cite{agarwal2021multiagent}, the same \gls{mab} problem is solved cooperatively by a set of agents that can share some information about the best estimated actions. In this case communication is peer-to-peer with a deterministic rate, and again there is no contextual information. A related formulation has been proposed in \cite{Kalkanli:nips:2021}, where the \textit{Batched Thompson Sampling} algorithm is introduced. In this case the goal is to reduce the number of policy updates, thus reducing the computational complexity of the algorithm. Rounds are grouped into batches, and within one batch arms are pulled without updating the sampling policy. Similarly, in our formulation one round can be considered as a batch, given that the $N$ agents operate in parallel. However, in \cite{Kalkanli:nips:2021}, the \textit{batch size}, i.e., the number of samples without policy updates, can be optimized by the algorithm and is not fixed during the training, i.e., different batches may have different sizes. On the contrary, in our case the batch size is fixed to $N$, and is given by the environment. Moreover, in \cite{Kalkanli:nips:2021} the authors do not consider the contextual case, and there are no communications constraints when pulling arms.

Interestingly, the authors of \cite{Lai2021} study a similar framework to justify the concept of information bottleneck in \gls{rl} \cite{Federici2020LearningRR, igl2019generalization} from a psychological perspective. In this work, the rate-constrained channel is compared to resource-limited policy storing systems, like the brain and noisy storage devices, commenting on the trade-off between policy complexity, i.e., level of correlation between states and actions, and policy performance, i.e., obtained reward. In these studies the single-agent case is considered, and the formulation is somewhat generic with qualitative and empirical considerations, without properly framing the underlying learning problem. In this older study \cite{collins_neuro_2012}, the authors performed a set of behavioral experiments that resemble the \gls{cmab} problem, recording correlations between performance variance, i.e., action stochasticity, and model capacity, i.e., work memory, in humans.

With this work, the aim is to study this relation within the proper information-theoretic setting, relating it to the learning task of \gls{cmab}, and  highlighting the theoretical trade-off between channel rate and learning performance. Moreover, we study practical ways to compress policies having limited impact in the training process.

\section{Problem Formulation}
\label{sec:problem_formulation}

\subsection{The Contextual Multi-Armed Bandit (CMAB) Problem}
\label{sub:mab}

The standard single-agent \gls{cmab} problem considers an agent interacting with the environment by pulling arms upon the observation of some contextual information, and receiving a reward based on the observed context and pulled arm. Specifically, at each round $t = 1, \dots, T$, the environment samples a context $s_t \in \mathcal{S}$ following distribution $P_S$, where $\mathcal{S}$ is a finite set containing all possible contexts. We assume that all contexts are observable, i.e., $\forall s \in \mathcal{S}, ~ P_S(s) > 0$. Given $s_t$, the agent chooses an arm $a_t \in \mathcal{A} = \{1, \dots, K\}$, with probability $\pi_t(a_t | s_t)$. Given the pair $(s_t, a_t)$, the environment returns a stochastic reward $R(s_t, a_t)$ sampled according to $P_R(r|s_t, a_r)$, which is an unknown and stationary distribution that characterizes the reward statistics, and depends on the sampled context and the arm pulled by the agent. We then define $\mu(s,a) = \mathbb{E}_{P_R} \left[ R(s, a)\right]$, and further assume that $R(s, a) \in \left[0, 1\right], ~ \forall s \in \mathcal{S}$ and $\forall a \in \mathcal{A}$. Moreover, we assume that the reward distributions belong to the exponential family\footnote{If $\bm{\theta} = (s, a)$, we can write $ P_R(r|\bm{\theta} ) = b(r) \exp( \eta(\bm{\theta}) T(r) - A(\bm{\theta}))$, where $b$, $T$, and $A$ are known functions, and $A(\bm{\theta})$ is assumed to be twice differentiable.}, as also detailed in \cite{approx_ts}, Assumption 1. The policy $\pi_t(a_t | s_t)$ employed by the agent is a map $\pi_{t} : \mathcal{H}^{t-1} \times \mathcal{S} \rightarrow \Delta_K$, where $\Delta_K$ denotes the $(K-1)$-simplex, containing all possible distributions over the set of $K$ arms, and $\mathcal{H}^{t-1}$ is the history, containing all possible observations, arms and rewards collected until round $t-1$, i.e., the element $H(t-1) \in \mathcal{H}^{t-1}$ is $H(t-1) = \left\{ s_1, a_1, r_1, \dots, s_{t-1}, a_{t-1}, r_{t-1} \right\}$. The goal for the agent is to optimize the policy $\pi_t$ in order to maximize the average sum of received rewards or, equivalently, to minimize the Bayesian regret 
\begin{equation}
\label{eq:bayesian_regret}
\text{BR}(\pi, T) = \mathbb{E} \left[  \sum_{t=1}^T \mu(s_t, a^*(s_t)) - \mu(s_t, A_t) \right],
\end{equation}
where $A_t$ is the arm pulled by the agent in round $t$ sampled according to $\pi_t(a|s_t)$, and $a^*(s_t) = \argmax_{a \in \mathcal{A}} \mu(s_t, a)$ is the optimal arm in round $t$, i.e., the one that maximizes the average reward in context $s_t$. If $a^*(s)$ is not unique, it represents an arbitrarily chosen arm among the optimal ones. Here the expectation is taken with respect to the state, action and problem instance distributions.

\subsection{Rate-Constrained CMAB }
\label{sub:rccmab}

We consider a system in which $N$ agents have to solve the same realization of a \gls{cmab} problem, where each agent observes an independent context, distributed according to $P_S$. The agents can only interact with the environment pulling arms, whereas the contexts are observed by a remote decision-maker. The decision-maker communicates with the agents through a controller, that is in charge of receiving instructions from the decision-maker and informing the agents accordingly. However, the decision-maker can transmit information to the controller through a rate-constrained channel, which imposes a constraint on the number of bits per agent the decision-maker can transit to the controller. Consequently, at each system round, the environment samples $N$ contexts $\left\{s_{j,n}\right\}_{n=1}^N$, i.e., one per agent, which are observed by the decision-maker, which, in turn, needs to decide and communicate to the controller the $N$ arms $\left\{a_{j,n}\right\}_{n=1}^N$ to be pulled by the agents. 

\begin{figure}[t!]
	\centering
	\includegraphics[width=1\linewidth]{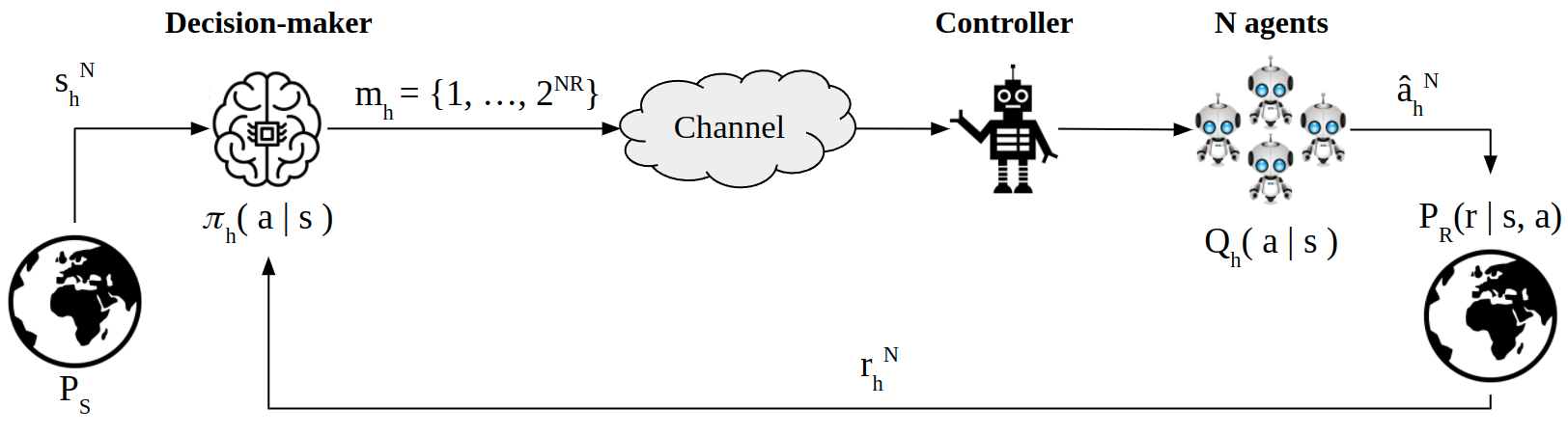}
	\caption{The \gls{rccmab} problem formulation.}
	\label{fig:rccmab}
\end{figure}

As in standard \gls{cmab}, the decision-maker exploits the knowledge accumulated until system round $j$, and encoded in the variable $H(j) = \left\{ \left\{s_{1,n}, a_{1,n}, r_{1,n} \right\}_{n=1}^N, \ldots, \left\{ s_{j,n}, a_{j,n}, r_{j,n}\right\}_{n=1}^N \right\} \in \mathcal{H}^{(j)}$ to optimize its policy $\pi_j$. However, in our setting, the decision-maker can interact with the controller only through a rate-constrained communication channel, which may not allow to transmit all the intended arms to the agents. Consequently, the problem is to communicate the arm distribution, i.e., the policy $\pi_j(a|s)$, which depends on the specific context realizations observed in round $j$, to the controller within the available communication resources while inducing the minimal impact on the performance of the learning algorithm. To this end, the decision-maker employs a function $f^{(N)}_j: \mathcal{H}^{(j-1)} \times \mathcal{S}^N \rightarrow \{1, 2, \ldots, B\}$ mapping the knowledge acquired up to round $j-1$, together with the agents' contexts, to a message index to be transmitted over the channel. At the receiver, the controller adopts a function $g^{(N)}_j: \{1, 2, \ldots, B\} \rightarrow \mathcal{A}^N$ to decode from the received message the $N$ arms to be pulled by the agents. In general, both functions $f^{(N)}_t$ and $g^{(N)}_t$ can be stochastic. We then define the Bayesian system regret as

\begin{equation}
\label{eq:bayesian_system_regret}
\text{BR}\left(J, \left\{f_j^{(N)}, g_j^{(N)}\right\}\right) =  \mathbb{E} \left[ \sum_{j=1}^{J} \sum_{n \in \mathcal{N}} r(s_{j,n}, a^*(s_{j,n}) - r(s_{j,n}, g_{j,n}(m_j)) \right],
\end{equation}
where $g_{j,n}(m_j)$ is the arm pulled by agent $n$ during round $j$ decoded from the message $m_j = f_j^{(N)} \left( H(j-1), s_j^N\right)$, and $s_j^N \in \mathcal{S}^N$ is the vector containing the contextual information in round $j$ for all the $N$ agents. The goal is to specify the encoding and decoding functions, $f^{(N)}_j$ and $g^{(N)}_j$, to minimize the Bayesian system regret in Eq.~(\ref{eq:bayesian_system_regret}). More specifically, the goal is to obtain a regret which is sub-linear in $J$, possibly achieving the same performances with standard \gls{cmab} solutions \cite{lattimore_szepesvári_2020}. For a problem with $N$ agents, a rate $R$ is said to be achievable if there exist functions $\left\{ f^{(N)}_j, g^{(N)}_j\right\}_{j=1}^J$ with rate $\frac{1}{N} \log_2 B \leq R$, and Bayesian system regret 

\begin{equation}
\label{eq:achievable_rate}
\lim_{J \rightarrow \infty} \frac{\text{BR}\left( J, \left\{f_j^{(N)}, g_j^{(N)}\right\} \right)}{J} = 0.
\end{equation}

Note that when there is no communication channel from the decision-maker to the controller, i.e., $B=0$, sub-linear regret is not possible since the agents cannot learn on their own without observing the contexts and rewards. On the other hand, when $B$ is sufficiently large, i.e., $B \geq N \log_2 K$, any desired arm sequence can be conveyed to the controller, and the problem becomes the standard centralized \gls{cmab} problem with $N$ parallel agents. Our goal is to identify the minimal communication needed from the decision-maker to the controller that makes sub-linear regret feasible.
We would like to emphasize that the introduced \gls{rccmab} problem differs from the standard \gls{cmab} formulation in two aspects: First, at each round $j$, the decision-maker pulls $N$ parallel arms through the agents exploiting $\pi_j$, which is updated at the end of each round. This is similar to the standard \gls{cmab} formulation in which the policy can be updated only every $N$ arm pulls. Second, given the available rate $R$, the decision-maker may not be able to convey the exact sequence of arms $\{a_{j,n}\}_{n=1}^N$ sampled from $\pi_j$, and instead must send a compressed version, which may result in some agents to pull sub-optimal arms. We highlight that this is a lossy compression problem; however, unlike classical lossy source coding problems, the goal is not to send a sequence of arms with highest average fidelity, but to enable the agents to pull the arms that would result in a sub-linear regret.

\section{Theoretical Limits}
\label{sec:asymptotic_limits}

In this section, we provide a theoretical analysis of the regret bound achievable by the \gls{ts} strategy, and the minimum rate required to achieve sub-linear regret.

\subsection{\gls{ts} Performance}
\label{subsec:ts_ucb_performance}
 
We first provide the regret performance of the \gls{ts} algorithm,  when there is no constraint on the available rate to transmit the intended arms. In this work, we analyze the general case in which no prior structure is assumed between the optimal policies and the different contexts, and so we consider the simplest implementation of one \gls{mab} agent for each context $s \in \mathcal{S}$. 
%
The \gls{ts} algorithm adopts a Bayesian strategy, estimating the distribution $p^{s,a}(\mu)$ of the reward mean $\mu(s, a) \in \left[0, 1\right]$ in each round $j$ with $p_j^{s,a}(\mu)$. When observing the context $s_j$, it samples $\hat{\mu}_j(s_j, a)  \sim p_j^{s,a}(\mu)$, $\forall a \in \mathcal{A}$, and pulls the arm $a_j = \argmax_{a \in \mathcal{A}} \{\hat{\mu}_j(s_j, a)\}$. In \gls{rccmab} the decision-maker adopts the described sampling strategy for each agent $n \in \{1, \dots, N\}$. After receiving all the rewards $\left\{r_{t,n}\right\}_{n=1}^N$, the decision-maker updates its belief on $\mu(s, a)$ optimizing the posteriors $p_j^{s,a}(\mu)$. The variance of the posterior distributions is exploited to perform exploration. This algorithm is well studied, and is known under the name of Thompson Sampling (TS) \cite{thompson}. The \gls{ts} algorithm implicitly induces a probability distribution $\pi_j(a | s)$ over the arms that can be computed as
\begin{align*}
	\pi_j(a|s) = \int_{\mathbb{R}} p_j^{s,a}(\mu) \prod_{k=1, k \neq a}^{K} P_j^{s,k}(\mu) d\mu,
\end{align*}
where $P_j^{s,k}(\mu)$ is the \gls{cdf} of $\mu(s, k)$, and the random variables $\mu(s, a)$ are independently distributed. We will call $\pi_j(a|s)$ the target policy, i.e., the one that the decision-maker would like to convey to the controller. However, in \gls{rccmab} the constraint on the rate of the communication channel may not allow to sample the arms according to $\pi_j(a | s)$, as explained in Sec.~\ref{sec:problem_formulation}. In this case, the problem is that the decision-maker updates the posteriors $p_j^{s,a}$, and so $\pi_j(a|s)$, using the \gls{ts} algorithm, but can only sample with an approximate policy $Q_j(a|s)$, whenever the available rate is not sufficient to convey $\pi_j(a|s)$.

\subsection{Regret Bounds}
\label{subsec:asymptotic_bounds}

We now report the performance of the \gls{ts} algorithm, when the available rate is sufficient to perfectly convey the arms from the decision-maker to the controller in each round $j = 1, \dots, J$. To this end, we align the parallel interactions between the agents and the environment in time, and consider virtual rounds $t \in \{ 1, \dots, J\cdot N\} = \mathcal{T} $, as if a single agent, i.e., the decision-maker, were playing a sequential \gls{cmab} game, with the constraint that the policy can be updated only every $N$ steps, i.e., at the end of each round. We observe that the order used to align the agent interactions does not affect the analysis, as the agents all play the same policy, and interact with \gls{iid} replicas of the same environment, leading the two formulations to be mathematically equivalent. Consequently, by providing the results as a function of the virtual rounds $t$, we consider the total number of interactions the agents have with the environment, which is consistent with the usual \gls{mab} notation.

\begin{theorem} [\gls{ts} Bayesian System Regret]
	\label{thm:sys_regret}
	The finite-time Bayesian system regret of \gls{ts} is upper bounded by
	\begin{equation}
	\text{BR}(\pi^{TS}, T) \leq 2SKN + 4\sqrt{\left(2 + 6\log T \right) SKN T},
	\end{equation}
	and the asymptotic regret is 
	\begin{equation}
	\text{BR}(\pi^{TS}, T) \in \mathcal{O} \left( \sqrt{KST\log T} \right).
	\end{equation}	
\end{theorem}
\begin{proof}
	See Appendix~\ref{proof:ts_regret}.
\end{proof}
We observe that, in the finite-time analysis, an additional term $\sqrt{N}$ appears, with respect to the standard single-agent performance \cite{lattimore_szepesvári_2020}. This is a consequence of the fact that, in one round $j$, $N$ arms are pulled in parallel, without updating the policy. This effect has highest impact during the first rounds, as the policy has not converged yet, and so sub-optimal arms are pulled in parallel. In the long term, the effect vanishes. This result is consistent with the analysis in \cite{Kalkanli:nips:2021}, with the difference that what the authors called batch, in our scenario is the parallel execution of the $N$ agents, and so in our case the batch size is fixed to $N$ and can not be optimized. Consequently, in the finite-time upper bound we obtain a factor $\sqrt{N}$, that replaces the factor $\sqrt{\alpha}$ obtained in \cite{Kalkanli:nips:2021}, where $\alpha$ is the so-called \textit{batch growing factor} \cite{Kalkanli:nips:2021}. The factor $S$ is introduced as we consider the \gls{cmab} problem.

%

\subsection{Rate-Distortion Function for Communicating Policies}
\label{subsec:rate_distortion}

We now present the minimum rate needed to transmit a policy, i.e., the arms $a_j^N = \left( a_{j,1}, \dots, a_{j,N}\right)$ to be sampled for each agent according to $\pi_j(a|s)$, and conditioned on the observed context vector $s_j^N = (s_{j,1}, \dots, s_{j,N})$, when a specific distortion function is adopted to measure the discrepancy between the sequence $z_j^N = \left((s_{j,1}, a_{j,1}), \dots, (s_{j,N}, a_{j,N})\right)$, and the sequence $\hat{z}_j^N = \left((s_{j,1}, \hat{a}_{j,1}), \dots, (s_{j,N}, \hat{a}_{j,N})\right)$, where $\hat{a}$ indicates the arms decoded by the controller based on the received message $m_j$, as indicated in Sec.~\ref{sub:rccmab}. In short, $\hat{a}_j^N$ is the vector containing the arms actually pulled by the agents. Given the underlying learning problem, the quality metric for vector $\hat{a}_j^N$ should not be based on a per-symbol distance, but rather on the sampling probability distributions. Indeed, to obtain sub-linear regret, what interests us is the probability of sampling specific sequences. Consequently, the distortion function $d(\hat{Q}_{\hat{z}^N}, P_{SA})$ compares the empirical distribution $\hat{Q}_{\hat{z}^N}$ of $\hat{z}^N$ with $P_{SA}$, which is the joint distribution $P_{SA} = P_S(S) \cdot \pi(A|S)$. In the sequel, we omit to explicitly write the round index $j$, as the analysis does not depend on it. We consider the particular case in which the distortion measure $d(\hat{Q}_{\hat{z}^N}, P_{SA})$ respects the following properties: it is $1)$ nonnegative; $2)$ upper bounded by a constant $D_{max}$; $3)$ continuous in $P_{SA}$ at $\hat{Q}_{\hat{z}^N}$; $4)$ convex in $P_{SA}$, and such that $5)$ $d(\hat{Q}_{\hat{z}^N}, P_{SA}) = 0 \iff \hat{Q}_{\hat{z}^N} = P_{SA}$. Given the assumptions above, the authors of \cite{CommDistribution} provide the rate-distortion function $R(D)$, i.e., the minimum rate $R=\frac{\log_2 B}{N}$ bits per symbol such that $ \E_{Q_{SA} }[d(\hat{Q}_{\hat{z}^N}, P_{SA})] \leq D$,  in the limit when $N$ is arbitrarily large. Here the expectation is taken with respect to the distribution $Q_{SA} = P_S(S) Q(A|S)$, where $Q(A|S)$ is the sampling policy decoded by the controller from the message sent by the decision-maker. The solution is given by
\begin{align}
\label{eq:rate_dist}
R(D) = \min_{Q_{A|S} :  d(Q_{SA}, P_{SA}) \leq D} I(S; A).
\end{align}

As we can see, in the asymptotic limit when $N \rightarrow \infty$, the problem admits a single-letter solution, which also serves as a lower bound for the finite agent scenario. Let $R_{\pi_j}$ denote the rate required to convey the policy $\pi_j$ perfectly, i.e., with zero distortion. From Eq.~(\ref{eq:rate_dist}), we can see that $R_{\pi_j} = I(S;A)$, where the mutual information is computed under $P_{SA}$ dictated by the policy.

\subsection{Achievable Rate}
\label{subsec:achievable_rate}

We now state the rate condition under which it is possible to achieve sub-linear regret. In the analysis, the available rate $R$ is considered fixed in each round $j$. First of all, we denote with $H(A^*)$ the entropy of the arms under the marginal distribution $\pi^*(a) = \sum_s P_S(s) \pi^*(a|s)$, where $\pi^*$ is the optimal policy, i.e., the one that selects, $\forall s \in \mathcal{S}$, the arm $a^*= \argmax_{a \in \mathcal{A}} \mu(s, a)$. We start by stating the following Lemma, which provides a rate limit below which it is not possible to achieve sub-linear regret.

\begin{lemma}\label{lemma:min_rate}
	If $R < H(A^*)$, it is not possible to achieve sub-linear Bayesian system regret.
\end{lemma}
\begin{proof} 
	See Appendix~\ref{proof:rate}.
\end{proof}
The following Lemma provides the achievability part. 
\begin{lemma}
	\label{lemma:achivable_rate}
	If $R > H(A^*)$, then it is possible to achieve sub-linear Bayesian system regret in the limit $N \rightarrow \infty$.
\end{lemma}
\begin{proof}
	See Appendix~\ref{proof:rate}.
\end{proof}

We can see that, due to Lemma~\ref{lemma:achivable_rate}, even if for some round $j$, $R_{\pi_j} > R$, as long as $R > H(A^*)$, it is still possible to achieve sub-linear regret. According to the definition in Eq.~(\ref{eq:achievable_rate}), this implies that, as $N \rightarrow \infty$, any rate $R > H(A^*)$ is achievable, while any rate $R<H(A^*)$ is not achievable. The entropy of the marginal $\pi^*(a)$ is thus the fundamental information-theoretic limit of the problem to achieve sub-linear regret.

\section{Policy Compression}
\label{sec:policy_compression}

We are now ready to study compression strategies to deal with the case in which it is not always possible to convey the policy with zero distortion, i.e., $\exists ~ j$ s.t. $R_{\pi_j} > R$. In such cases, it is not clear which is the message $m_j$ the decision-maker should transmit to the controller. As a consequence of Eq.~(\ref{eq:rate_dist}), when $R_{\pi_j} > R$, the sampling policy $Q_j$ adopted by the controller may differ from $\pi_j$. In \cite{approx_ts}, the authors provide some theoretical guidelines to construct approximate sampling policies to make the posteriors, i.e., $\pi(a|s)$, converge to the optimal one achieving sub-linear regret, even when an agent is sampling with a different policy $Q(a|s)$. In particular, they studied the case in which the sampling distribution $Q$ differs from the target posterior $\pi$, using the $\alpha$-divergence $D_{\alpha}(\pi, Q)$ as distortion measure, which is defined as
\begin{align}
D_\alpha (\pi,Q) = \frac{1 - \int \pi(x)^\alpha Q(x)^{1 - \alpha} dx}{\alpha(1-\alpha)}.
\end{align}

We now provide two theoretical compression schemes that adopt the forward KL divergence $ D_{\alpha \rightarrow 0}(\pi, Q) = D_{KL}(\pi || Q)$, and the reverse KL divergence $ D_{\alpha \rightarrow 1}(\pi, Q) = D_{KL}(Q || \pi)$, as distortion functions, which are the two cases considered also in \cite{approx_ts}. We remember that, for two discrete distributions $p$ and $q$ such that $p$ is absolutely continuous with respect to $q$, i.e., if $x \in \mathcal{X}$ is such that $q(x)=0$, then $ p(x) = 0$, $D_{KL}(p || q) = \sum_{x \in \mathcal{X}} p(x) \log \frac{p(x)}{q(x)}$. 

Another reason to adopt these two metrics is related to the fact that it is possible to bound, in each round $j$, the gap between the expected reward of the target policy $\pi_j$, and that obtained by using the approximate policy $Q_j$. To this end, we denote by $\mu^{\pi}(s,a)$ the average reward obtained in context $s$ when using policy $\pi$, i.e., $\mu^{\pi}(s,a) = \mathbb{E}_{\pi(a|s)}\left[ \mu(s, a)\right]$, and find
\begin{align}
	\mathbb{E}_{P_S} \left[ \left| \mu^{\pi_j}(S, A) - \mu^{Q_j}(S, A) \right| \right] &=\sum_{s \in \mathcal{S}} P_S(s) \sum_{a \in \mathcal{A}} \mu(s, a) \lvert \pi_j(a|s) - Q_j(a|s)\rvert\\
	& \stackrel{(a)}{\leq} \sum_{s \in \mathcal{S}} P_S(s) \sum_{a \in \mathcal{A}} \lvert \pi_j(a|s) - Q_j(a|s)\rvert \\
	&=  \sum_{s \in \mathcal{S}} P_S(s) \lvert\lvert \pi_j(\cdot |s) - Q_j(\cdot |s)\rvert\rvert_1 \\
	& \stackrel{(b)}{\leq} C \cdot \mathbb{E}_{P_S} \left[ \sqrt{ D_{KL}\left(\pi_j(\cdot|S) || Q_j(\cdot|S) \right)}\right],
\end{align}
where (a) holds because $\mu(s, a) \in \left[0, 1\right]$ by assumption, (b) is the Pinsker's inequality, and $C$ is a constant that depends on the base of the logarithm in the divergence, e.g., $C = \sqrt{\frac{1}{2 \ln 2}}$ if in base $2$. We notice that, by swapping the roles of the two distributions in the last inequality, it is possible to obtain the version with the reverse KL divergence.

\textbf{Observation.} It is known that by minimizing the forward KL divergence $D_{KL}(\pi || Q)$, we obtain a more "spread" solution $Q$, that tends to cover the whole domain of $\pi$. Indeed, $\forall x \in \mathcal{X}$ s.t. $\pi(x) >0$, a penalty is added whenever the two distributions differ, i.e., the function $ \log\frac{\pi(x)}{Q(x)}$ is weighed by $\pi(x)$. On the other hand, in the reverse KL divergence, the log function is weighed by distribution $Q$. This means that, if $Q(x) =0$, there is no incurred penalty for not approximating the policy $\pi(x)$ in $x$, leading to a solution that puts mass around the peaks of $\pi$. Consequently, the exploration-exploitation trade-off is usually biased towards a more exploration-seeking solution when minimizing $D_{KL}(\pi || Q)$, and to a more exploitation-seeking solution when minimizing $D_{KL}(Q|| \pi)$.

\subsection{Reverse KL Divergence}
\label{subsec:compression_rev_kl}

In this case, the adopted distortion function is $d(Q_{SA}, \pi_{SA}) = \mathbb{E}_{P_S} \left[ D_{KL}(Q(\cdot | S) || \pi(\cdot | S))\right]$. In the following lemma, we provide the shape of the policy $Q(a|s)$ that can achieve the minimum in Eq.~(\ref{eq:rate_dist}), i.e., the policy that minimizes the required rate while meeting the constraint on the reverse KL divergence.

\begin{lemma}\label{lemma:compression_rev_kl}
	Given the constraint on the reverse KL divergence $\text{D}_{KL} \left( Q || \pi \right) \leq \delta$, the policy that achieves the minimum in Eq.~(\ref{eq:rate_dist}) is
	\begin{equation}
		Q_{\lambda}(a|s) = \frac{\tilde{Q}(a)^{\lambda} \pi(a|s)^{1-\lambda}}{Z},
	\end{equation}
	where $\lambda \in \left[ 0, 1\right]$ is such that $\text{D}_{KL} \left( Q_{\lambda} || \pi \right) = \delta$, $\tilde{Q}(a)$ is the marginal, i.e., $\tilde{Q}(a) = \sum_{s \in \mathcal{S}} P_S(s) Q_{\lambda}(a|s)$, and $Z$ is the normalization factor.
\end{lemma}
\begin{proof} 
	See Appendix~\ref{subsub:reverse_kl_proof}.
\end{proof}

We can see that, when $\lambda = 0$, we have $Q(a|s) = \pi(a|s)$, and so $R_Q = R_{\pi} = I(S;A)$, where the mutual information is computed with respect to $P_{SA}$. As $\lambda$ increases from $0$ to $1$, the policy $Q$ tends to be more similar to the marginal $\pi(a) = \sum_{s \in \mathcal{S}} P_S(s) \pi(a|s)$, which requires zero rate. Indeed, when $\lambda = 1$, the marginal $\pi(a)$ does not require context information, and so the agents would sample arms regardless of the contexts $s_{j,n}$.

\subsection{Forward KL Divergence}
\label{subsec:compression_forward_kl}

We consider $d(Q_{SA}, \pi_{SA}) = \mathbb{E}_{P_S} \left[ D_{KL}( \pi(\cdot | S) || Q(\cdot | S) )\right]$. As for the reverse case, we report the shape of the optimal compressed policy.

\begin{lemma}\label{lemma:compression_forward_kl}
	Given the constraint on the forward KL divergence $\text{D}_{KL} \left(\pi || Q \right) \leq \delta$, the policy that achieves the minimum in Eq.~(\ref{eq:rate_dist}) is
	\begin{equation}
	Q_{\lambda}(a|s) = \frac{\lambda \pi(a|s) ~ \mathbf{W}_0\left( \frac{\lambda \pi(a|s)}{\tilde{Q}}\right)}{Z}
	\end{equation}
	where $\lambda$ is such that the maximum distortion is achieved with equality $\text{D}_{KL} \left( \pi || Q_{\lambda}\right) = \delta$, $\mathbf{W}_0(\cdot)$ is the Lambert function \cite{lambert}, $\tilde{Q}(a)$ is the marginal, i.e., $\tilde{Q}(a) = \sum_{s \in \mathcal{S}} P_S(s) Q_{\lambda}(a|s)$, and $Z$ the normalization factor.
\end{lemma}
\begin{proof} 
	See Appendix~\ref{subsub:kl_proof}.
\end{proof}

We observe that the above compressed policies can achieve the minimum in Eq.~(\ref{eq:rate_dist}) as $N \rightarrow \infty$, characterizing the information-theoretic limit and serving as a lower bound for more practical schemes that can work for finite $N$.

\textbf{Remark.} We notice that, in general, $D_{KL}(p||q)$ always satisfies conditions $1)$ and $3)-5)$ defined in Sec.~(\ref{subsec:rate_distortion}), but not condition $2)$, i.e., it may be unbounded. We now define $\sigma_q := \min_{x \in \mathcal{X}: q(x) >0} q(x)$. If $p$ is absolutely continuous with respect to $q$, by the \textit{inverse Pinsker's inequality} we have
\begin{align*}
	D_{KL}(p||q) \leq \frac{2 \text{TV}(p,q)^2}{\sigma_q} \leq \frac{8}{\sigma_q} = D_{max},
\end{align*}
where $\text{TV}(p,q)$ indicates the total variation distance between the distributions $p$ and $q$. To meet this requirement, we further assume $\exists \epsilon >0 : \pi(a|s) \geq \epsilon$, $\forall s \in \mathcal{S}$ and $a \in \mathcal{A}$, whenever the compression scheme has to be applied. This additional assumption is reasonable in our context, as the target policy $\pi_j$ needs exploration during training. Once $\pi_j$ has converged to the optimal policy (see Appendix~\ref{proof:rate}), further compression would lead to sub-linear regret.

\subsection{Practical Coding Scheme}
\label{sub:coding_scheme}

To find practical coding schemes, we propose a solution that is based on the idea of context reduction, and computes compact context representations. In essence, the decision-maker constructs a message containing the new context representations $\hat{s}(s) \in \hat{\mathcal{S}}$ of $s$, one for each agent, and sends it over the channel. Once the agents have received the message, they can sample the arms according to a common policy $Q_{\hat{s}}(a|\hat{s})$, which is defined on the compressed context space $\mathcal{\hat{S}}$. If the rate constraint imposes $B$ bits per agent, it means that it is possible to transmit at most $2^B$ different contexts to each agent. The idea is to group the contexts into $2^B = M$ clusters $\hat{s}_1, \dots, \hat{s}_M$, minimizing $d(Q_{\hat{S}A}, P_{SA})$, where $Q_{\hat{S}A}$ is the new policy defined on the compressed contexts $\hat{s}(s) \in \hat{\mathcal{S}}$. Again, we avoid to explicitly write  the round index $j$, as the scheme does not depend on it.

To find the clusters and relative policies we employ the well-known Lloyd algorithm \cite{lloyd}, which is an iterative process to group states into clusters. First of all, knowing the policy $\pi$, the decision-maker maps each state $s \in \mathcal{S}$ into a $K$-dimensional point $\bm{\alpha^p} = \pi(\cdot|s) \in \Delta_K$, finding $|\mathcal{S}| = S$ different points $\bm{\alpha^1}, \dots, \bm{\alpha^S}$. Then, it generates $2^B=M$ random points $\bm{\mu^1}, \dots, \bm{\mu^M} \in \Delta_K$ as initial centroids, i.e., representative policies, and iterates over the following two steps:

\begin{enumerate}
	\item Assign to each point $\bm{\alpha^{p}}$ the class 
	$c^* \in \{1, \dots, M\}$ such that $c^* = \arg\min_c \mathbb{E}_{P_S} \left[D_{\alpha}(\bm{ \mu^c}, \bm{\alpha^p} )\right]$, i.e., minimizing the average $\alpha$-divergence between the representative $\bm{\mu^{c^*}}$ and the original policy, which is the point $\bm{\alpha^p}$. For each cluster $c$, we now denote by $\mathcal{S}_c$ the set containing the contexts associated to the policies in that cluster.
	\item Update $\bm{\mu^1}, \dots, \bm{\mu^M}$ such that
	$
	\bm{\mu^c} = \argmin_{\bm{\mu} \in \Delta_K} \sum_{s \in \mathcal{S}_c} P(s) D_{\alpha}( \bm{\mu}, \pi( \cdot | s)),
	$
	which is still a convex optimization problem when using $\alpha = 0$ or $\alpha = 1$, and can be solved applying the Lagrangian multipliers. The solution is 
	\begin{equation}\label{eq:centroids_rev_kl}
	\bm{\mu^c} = \frac{\prod_{s \in \mathcal{S}_c} \pi(\cdot | s)^{\frac{P(s)}{A\left(\mathcal{S}_c\right)}}}{Z},
	\end{equation}
	when using $D_{\alpha \rightarrow 0}(\bm{ \mu^c}, \pi(\cdot | s) ) = D_{KL}(\bm{ \mu^c} || \pi(\cdot | s))$, and
	\begin{equation}\label{eq:centroids_kl}
	\bm{\mu^c} = \frac{\sum_{s \in \mathcal{S}_c}P(s)  \pi(\cdot | s)}{Z},
	\end{equation}
	when using $D_{\alpha \rightarrow 1}(\bm{ \mu^c}, \pi(\cdot | s) ) = D_{KL}(\pi(\cdot | s) || \bm{ \mu^c})$. Here, the product is to be considered element-wise, $A\left(\mathcal{S}_c\right)$ is the sum of the contexts probabilities in $\mathcal{S}_c$, i.e., $A\left(\mathcal{S}_c\right) = \sum_{s \in \mathcal{S}_c} P(s)$, and $Z$ is the normalizing factor. After computing the new centroids, we go back to step 1). The derivations of Eq.~(\ref{eq:centroids_rev_kl}) and Eq.~(\ref{eq:centroids_kl}) are provided in Appendix~\ref{subsub:reverse_kl_proof} and Appendix~\ref{subsub:kl_proof}, respectively.
\end{enumerate} 

The process continues until the new solution does not decrease the average distortion between the cluster policies and the target ones.

\textbf{Observation.} Note that the controller is assumed to know the $M$ policies from which it samples the arms of the agents. This can be transmitted at the beginning of each round. In this case, the scheme is efficient as long as $N\log_2K \gg BP \log_2 K$, where $P$ is the number of bits used to represent the values of the \gls{pmf} $Q_{\hat{s}}(\cdot | \hat{s})$. For this reason, we provide a scheme where the new policy is updated not at every transmission, but only when the new target $\pi$ has changed considerably. In particular, if we denote with $\pi^{cls}$ the policy defined over the compressed context representation, with $\pi^{last} $ the last policy used to compute $\pi^{cls}$, and with $\pi$ the updated target policy, we compute and transmit $\pi^{cls}$ every time $D_{\alpha}(\pi^{last}, \pi)$ exceeds a threshold $\zeta$. 

\subsection{Policy Compression, Information Bottleneck and Maximum Entropy Reinforcement Learning}
\label{sub:ib_rl_me}

We discuss here the connection of the compression schemes presented above with two popular training strategies in \gls{rl}: the \textit{\gls{ib}} approach \cite{tishby_2000}, and \textit{\gls{merl}} \cite{Haarnoja2018}.

The \gls{ib} method is a popular way to train an \gls{rl} agent to maximize a target cumulative reward, with an additional regularization term which accounts for mutual information $I(S;A)$ between the states and actions. As done in \cite{goyal2018transfer, igl2019generalization, Federici2020LearningRR}, the target function $J(\pi)$ to be maximized with respect to policy $\pi$ is
\begin{align}
\label{eq:ib}
	J^{IB}(\pi) = \mathbb{E} \left[ R^{\pi}\right] - \beta I(S;A) = \mathbb{E} \left[ R_{\pi}\right] - \beta H(A) + \beta H(A|S) ,
\end{align}
inducing the policy $\pi$ to forget task-independent state information, achieving a better generalization performance. 

A different but similar concept is that of \gls{merl}, in which the \gls{rl} agent is trained to maximize the reward, together with a regularization term that is the entropy of the conditional policy given the state \cite{Haarnoja2018, Levine2018}, i.e., 
\begin{align}
\label{eq:merl}
	J^{MERL}(\pi) = \mathbb{E} \left[ R^{\pi}\right] + \beta H(A|S),
\end{align}
inducing the policy $\pi$ to more exploration. As proved in \cite{Levine2018}, \gls{merl} naturally appears in the \textit{Control as Inference} framework, in which a probabilistic view of the \gls{rl} problem is adopted. Specifically, when using a uniform prior distribution over the actions, inferring the optimal posteriors leads to the optimization object in Eq.~(\ref{eq:merl}). 

If we now write the distortion-rate function of the \gls{rccmab} problem, and solve it using the Lagrangian multiplier method, we end up minimizing the objective 
\begin{align}
	\mathcal{L}^{RC-CMAB}(Q) = d(Q_{SA}, \pi_{SA}) + \lambda I(S;A),
\end{align}
where the term $d(Q_{SA}, \pi_{SA})$ plays the role of $\mathbb{E} \left[ R^{\pi}\right]$, i.e., it is related to reward maximization by playing according to the optimized posterior, and $\lambda I(S;A)$ is implied by the constraint on the rate. However the difference is that, in our analysis, the parameter $\lambda$ is optimized to meet the rate constraint in each round, and so it is determined by the problem, rather than being a hyper-parameter to be fine tuned as in the other two formulations. Moreover, in Sec.~\ref{subsec:achievable_rate} we proved that, if $\lambda$ is such that $R_{Q_{\lambda}} = I_{Q_{\lambda}}(S;A) < R_{\pi^*} = H(A^*)$, then the learning algorithm cannot achieve sub-linear regret. Consequently, it is important to carefully tune the coefficient $\beta$ when optimizing the policy with the \gls{ib} and \gls{merl} methods. This connection will be also highlighted by the results of the experiments in Sec.~\ref{sec:numerical_results}.

\section{Numerical results }
\label{sec:numerical_results}

We now analyze the \gls{rccmab} problem presented in Sec.~\ref{sec:problem_formulation}, and apply both the theoretical and practical policy compression schemes to solve it. In particular, we compare the performance of the \textit{Perfect} agent, which applies \gls{ts} without any rate constraint, and thus admits samples from the true posterior $\pi$, with the performance of the \textit{rate-constrained} algorithms \textit{Comm R-KL}, \textit{Comm F-KL}, \textit{Cluster R-KL}, and the \textit{Cluster F-KL} agents. The \textit{rate-constrained} agents adopt \gls{ts} at the decision-maker and, when the rate is not sufficient, transmit a compressed policy $Q$ with the different schemes. Specifically, the \textit{Comm R-KL} agent uses the optimal scheme with reverse KL divergence, provided in Sec.~\ref{subsec:compression_rev_kl}; the \textit{Comm F-KL} agent uses the forward KL divergence as explained in Sec.~\ref{subsec:compression_forward_kl};
the \textit{Cluster R-KL} and \textit{Cluster F-KL} agents implement the practical coding scheme provided in Sec.~\ref{sub:coding_scheme}. In every experiment, the context distribution $P_S$ is uniform over the $S=16$ contexts, and there are $K=16$ feasible arms, $N=50$ agents, and the total number of system rounds is $J=400$. In all the experiments the environment is such that, for each context $s \in \mathcal{S} = \left\{ 0, \dots, 15 \right\}$, the best average reward is given by the arm $a \in \mathcal{A} = \left\{ 0, \dots, 15 \right\}$ such that $a = \lfloor \frac{s}{G} \rfloor$, where $G$ is an experiment parameter. In particular, the reward behind arm $a$ with context $s$ is a Bernoulli random variable with parameter $\mu(s, a) = 0.8$ if $a = \lfloor \frac{s}{G} \rfloor$, and $\mu(s,a)$ is sampled uniformly in  $[0, 0.65]$ otherwise. The average rewards for sub-optimal arms are randomly generated at the beginning of each experiment's run.

This set of parameters allows us to study the performance of the different compression schemes, as the degree of correlation between optimal action distributions and contexts varies. We do not expect notable changes by varying the number of arms and/or reward distributions. For example, increasing the number of arms, or decreasing the gaps between optimal and sub-optimal arms' average rewards, would lead to longer training processes for any algorithms \cite{lattimore_szepesvári_2020}, without affecting the performance of compressed policies when compared to uncompressed ones.

\subsection{Optimal Rate Constraint}
\label{subsec:experiment_optimal_rate}

\begin{figure}[t]
	\centering
	\begin{subfigure}[b]{0.49\textwidth}
		\centering
		\includegraphics[width=\textwidth]{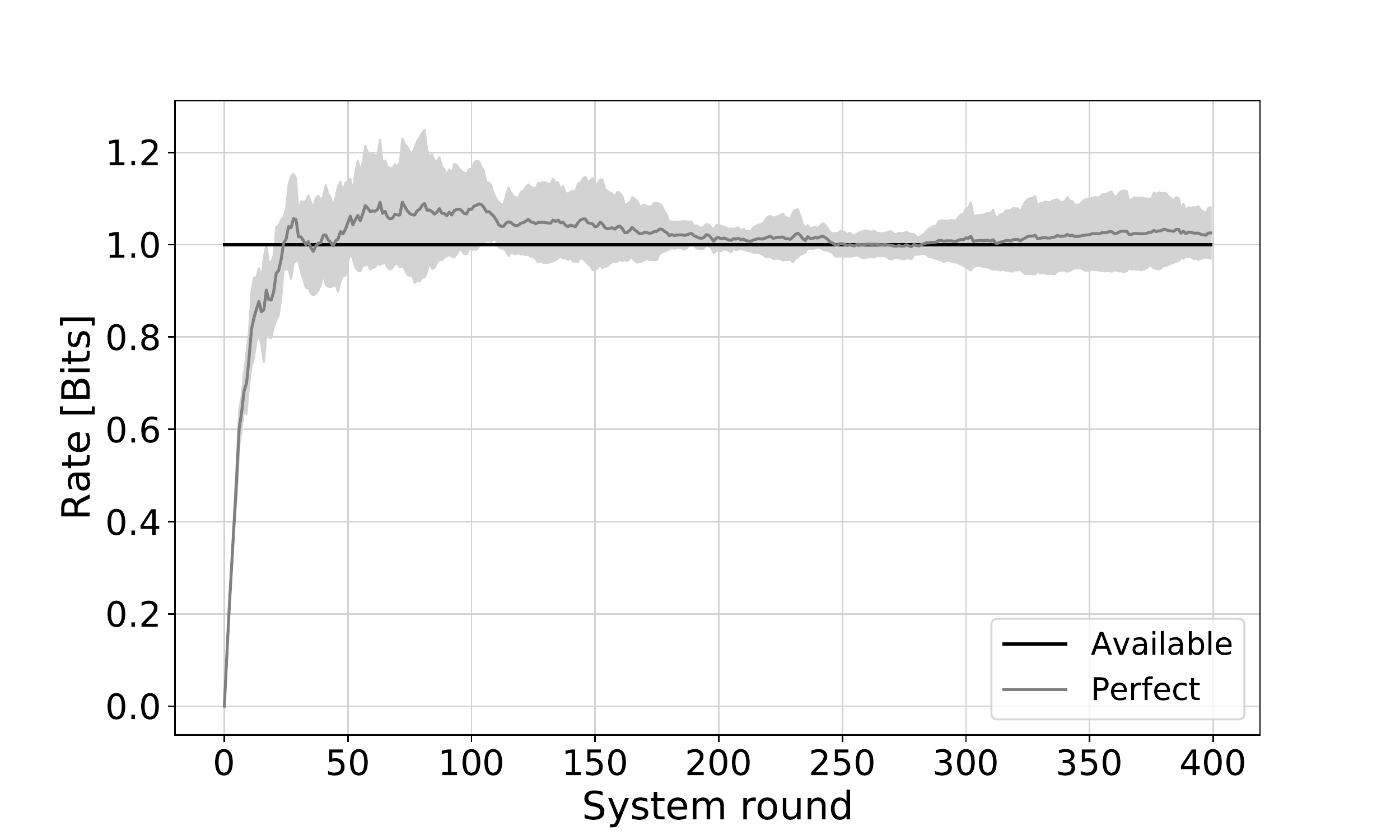}
		\caption{}
		\label{fig:rate_bottleneck_8}
	\end{subfigure}
	\hfill
	\begin{subfigure}[b]{0.49\textwidth}
		\centering
		\includegraphics[width=\textwidth]{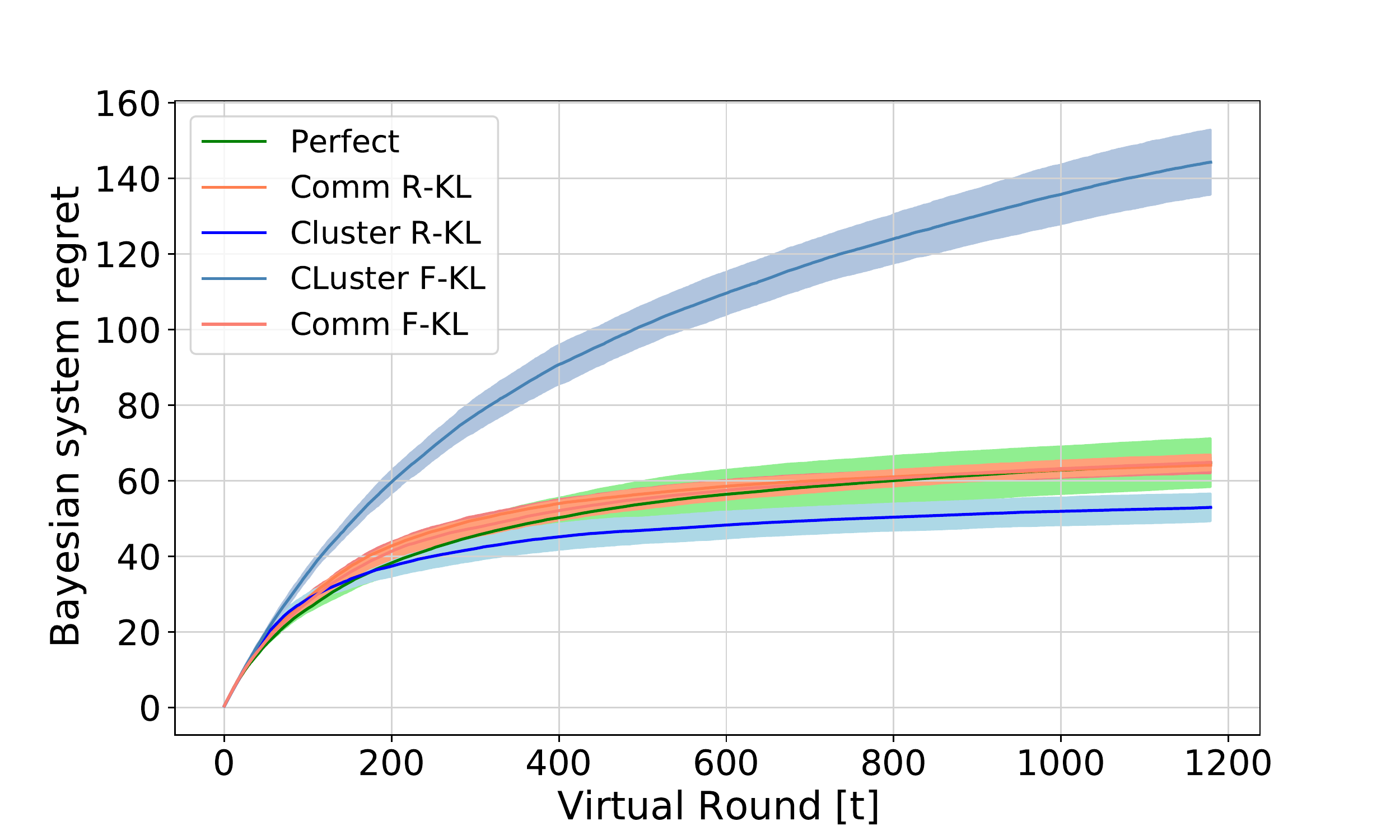}
		\caption{}
		\label{fig:regret_bottleneck_8}
	\end{subfigure}
	\caption{Rate $R_{\pi}$ needed to reliably transmit the policy $\pi$ as a function of the system round $j$, together with the imposed rate $R$ (a); average state regret obtained by the different algorithms as a function of the virtual round $t$ (b). In this case, $G=8$. Curves indicates average rewards $\pm$ one standard deviation over $5$ runs.}
	\label{fig:bottleneck_8}
\end{figure}

In the first experiment, we set $G=8$, meaning that the arm with the highest expected reward is $a=0$ in the first $8$ contexts, and $a=1$ with $s \in \left\{ 8, \dots, 15 \right\}$. Consequently, the rate of the optimal policy is $R_{\pi^*} = H(A^*) = 1$ bit, given that $P_S$ is uniform. The maximum available rate $R$ is set to $1$ bit for all the agents but the \textit{Perfect} one. Specifically, the agents \textit{Comm R-KL} and \textit{Comm F-KL} adopt a rate $R_{Comm} = \min \left\{ 1, R_{\pi}\right\}$. Whenever $R_{\pi} > 1$, the compression schemes explained in Sec.~\ref{sec:policy_compression} are adopted. The cluster agents use $1$ bit in all the rounds, whereas the \textit{Perfect} agent can always use $R_{\pi}$. The results are presented in Fig.~\ref{fig:bottleneck_8}. 

As we can see from Fig.~\ref{fig:rate_bottleneck_8}, as the learning process goes on, the required rate to transmit the \gls{ts} policy increases, as the mutual information between $S$ and $A$ increases. Moreover, the rate $R_{\pi}$ converges to $1$. However, it is interesting to observe Fig.~\ref{fig:regret_bottleneck_8}, that reports the average state regret as a function of the virtual round index $t = 1, \dots, JN$. The agents \textit{Comm R-KL}, \textit{Comm F-KL} and \textit{Perfect} all achieve similar performance. However, agent \textit{Cluster F-KL} is not able to converge to the optimal policy. Surprisingly, the best performing agent is the more practical \textit{Cluster R-KL}. We argue that this is because it is the agent that makes better use of the \gls{ib} trick, as $16$ different policies are not necessary to represent the optimal responses, given that $2$ different distributions are sufficient. We show thus that, when the hyper-parameter $\beta$ in Eq.~(\ref{eq:ib}) is optimally tuned, i.e., such that $I(S;A) = H(A^*)$, it is possible to gain in performance when adding the regularization term.

\subsection{Training with Rate Constraint}
\label{subsec:experiment_rate_constraint}

\begin{figure}[t]
	\centering
	\begin{subfigure}[b]{0.49\textwidth}
		\centering
		\includegraphics[width=\textwidth]{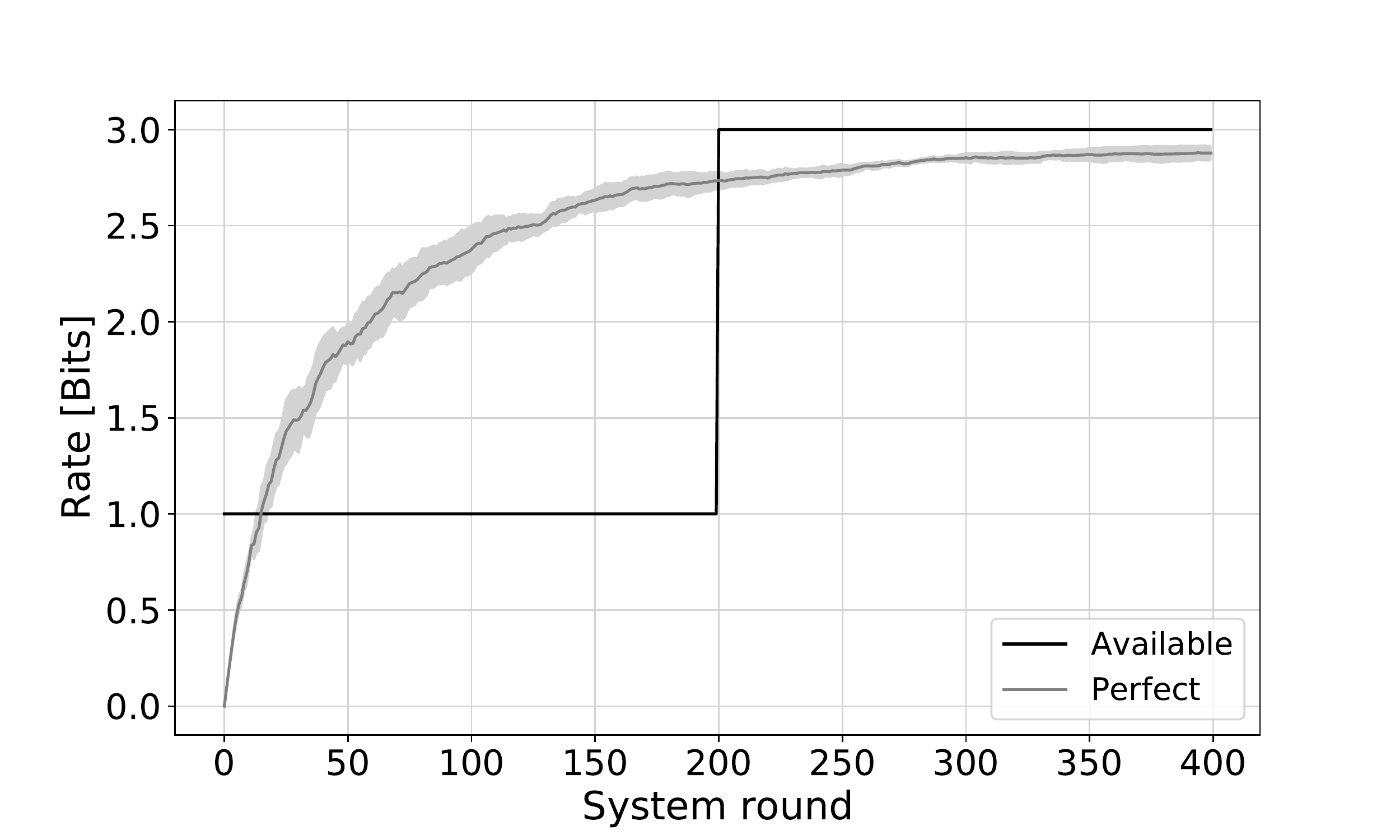}
		\caption{}
		\label{fig:rate_bottleneck_2}
	\end{subfigure}
	\hfill
	\begin{subfigure}[b]{0.49\textwidth}
		\centering
		\includegraphics[width=\textwidth]{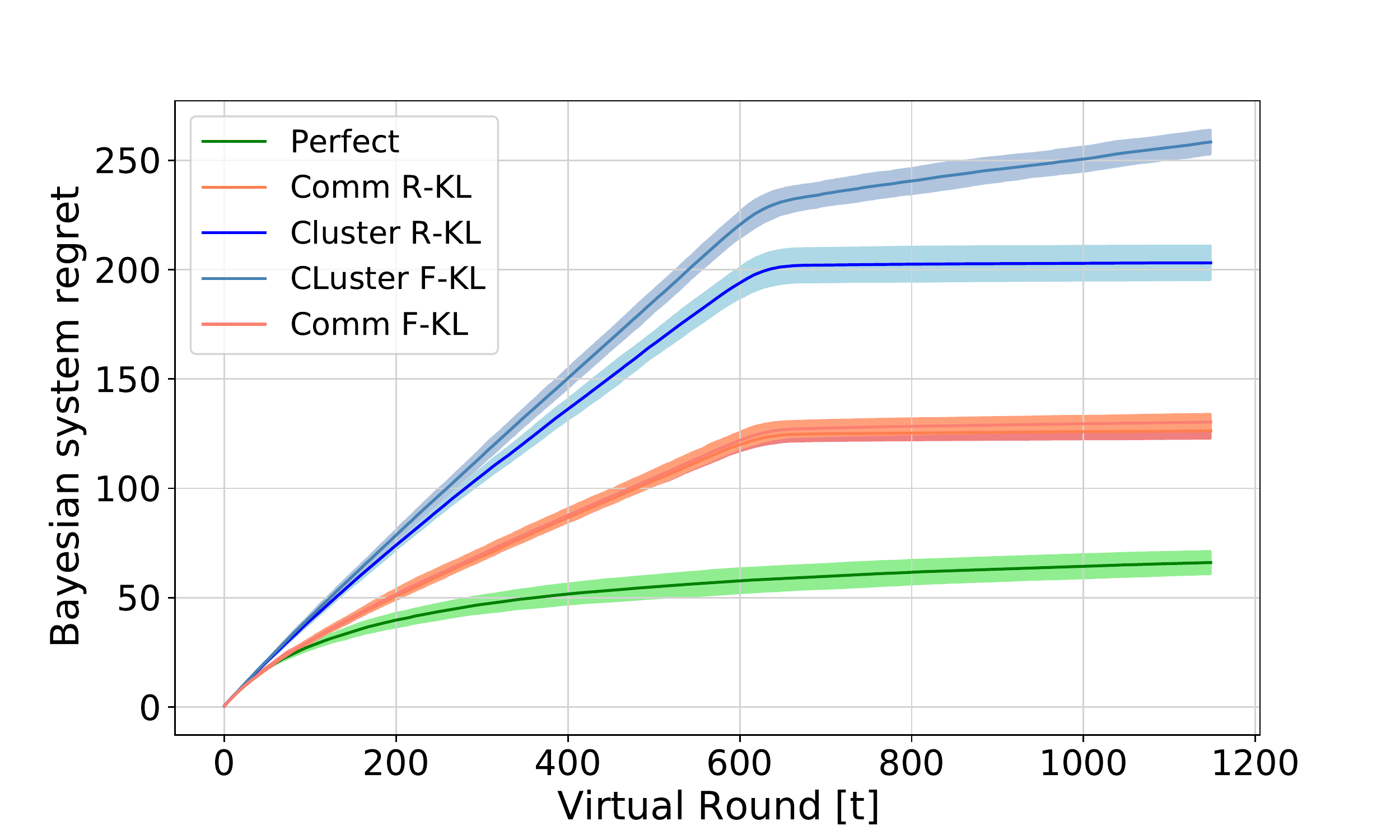}
		\caption{}
		\label{fig:regret_bottleneck_2}
	\end{subfigure}
	\caption{Rate $R_{\pi}$ needed to reliably transmit the policy $\pi$ as a function of the system round $j$, together with the imposed rate $R$ (a); average state regret obtained by the different algorithms as a function of the virtual round $t$ (b). In this case, $G=2$. Curves indicates average rewards $\pm$ one standard deviation over $5$ runs.}
	\label{fig:bottleneck_2}
\end{figure}

In this second experiment, we set $G=2$, and $R=2$ for the first $200$ system rounds, and $R=3$ for the remaining $200$ rounds. Given $G=2$, we now have $R_{\pi^*} = 3$ bits, and so the constraint can potentially damage the training process of the rate-constrained agents. As we can see from Fig.~\ref{fig:regret_bottleneck_2}, the \textit{Perfect} agent can easily converge obtaining sub-linear regret. The regret of the rate-constrained agents grows linearly as their rate is imposed to $R=1$. However, we can see that the two agents that theoretically-optimally trade rate with policy distortion, present the smallest slope in the regret curve. Again, the \textit{Cluster R-KL} agent outperforms \textit{Cluster F-KL}, and achieves almost zero per-round regret as soon as $R$ jumps to $3$, meaning that the posteriors converged to the optimal ones even when sampling with $R=1$. This is not true for \textit{Cluster F-KL}, as it presents larger slope in the regret curve, and the regret keeps growing even when $R=3$. These observations are consistent with the analysis in \cite{approx_ts}.

\subsection{Practical Coding Scheme}
\label{subsec:experiment_coding_scheme}

This last experiment presents the effectiveness of the practical coding schemes described in Sec.~\ref{sub:coding_scheme}, in the most complex case $G=1$, i.e., the association between context and optimal arm is a one-to-one map, and the optimal policy $\pi^*$ needs $R_{\pi^*} = 4$, which is also equal to the context entropy. Consequently, in this case the maximum-complexity policy, i.e., a different arm distribution for each context, is needed \cite{Lai2021}, and so the \gls{ib} should be carefully used, and $\beta = 0$ in the end is needed, as no compression scheme can achieve sub-linear regret with $R < 4$ bits. In this case, the number of bits the \textit{Cluster} agents can use is set to $R_{Cluster} = \lceil R_{\pi} \rceil$. What is interesting to notice here is that, in the long run, the \textit{Cluster F-KL} agent outperforms \textit{Cluster R-KL}, as it is able to converge to the optimal scheme by introducing more exploration. Indeed, even if in the first $400$ virtual rounds the \textit{Cluster R-KL} regret curve is below that for \textit{Cluster F-KL}, from $t \sim 600$ the \textit{Cluster F-KL} agent presents better performance. The reason is that, in this case, more exploration should be encouraged in the first steps to discover optimal policies, as done in \gls{merl}, by sampling sub-optimal arms more frequently. 

\begin{figure}[t]
	\centering
	\begin{subfigure}[b]{0.49\textwidth}
		\centering
		\includegraphics[width=\textwidth]{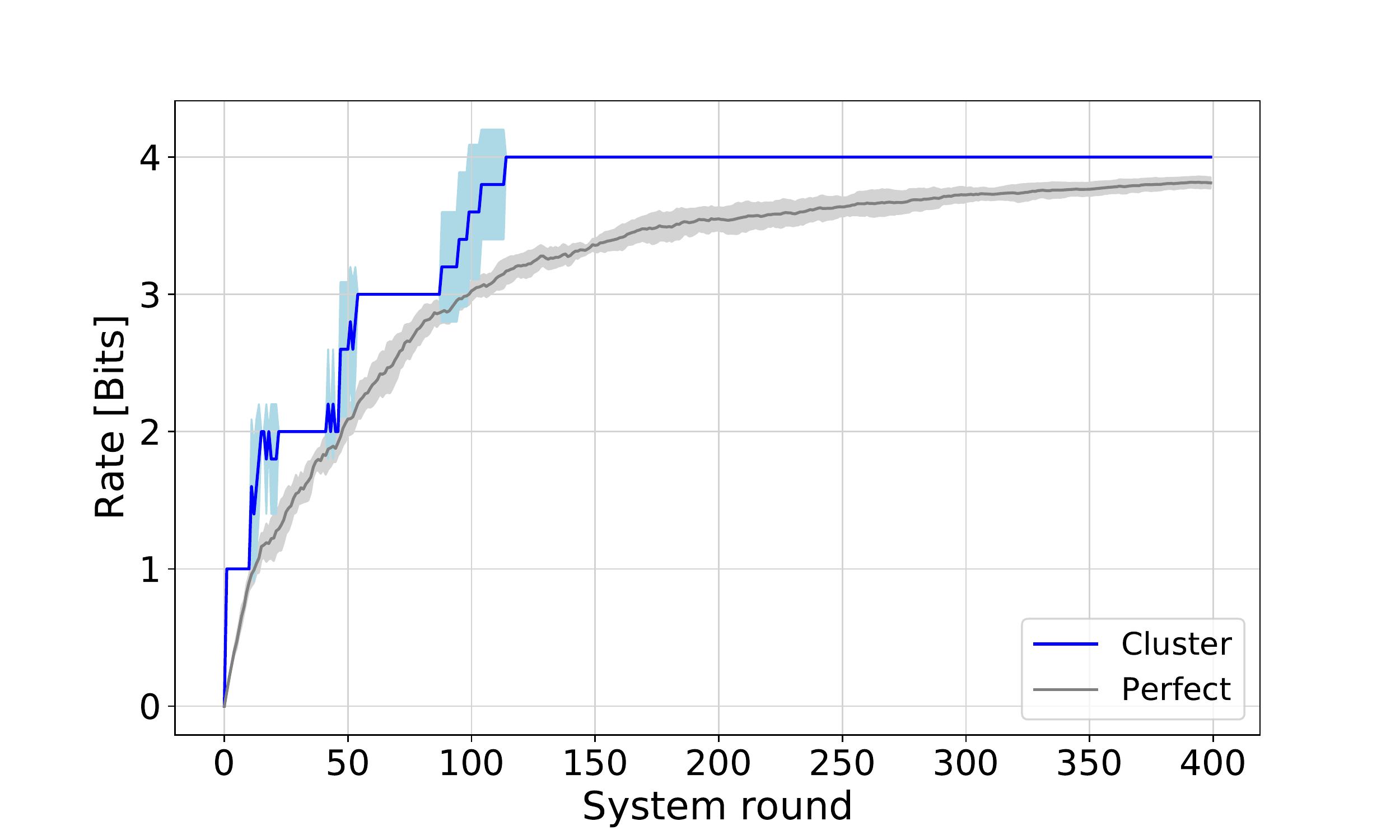}
		\caption{}
		\label{fig:rate_practical_scheme_1}
	\end{subfigure}
	\hfill
	\begin{subfigure}[b]{0.49\textwidth}
		\centering
		\includegraphics[width=\textwidth]{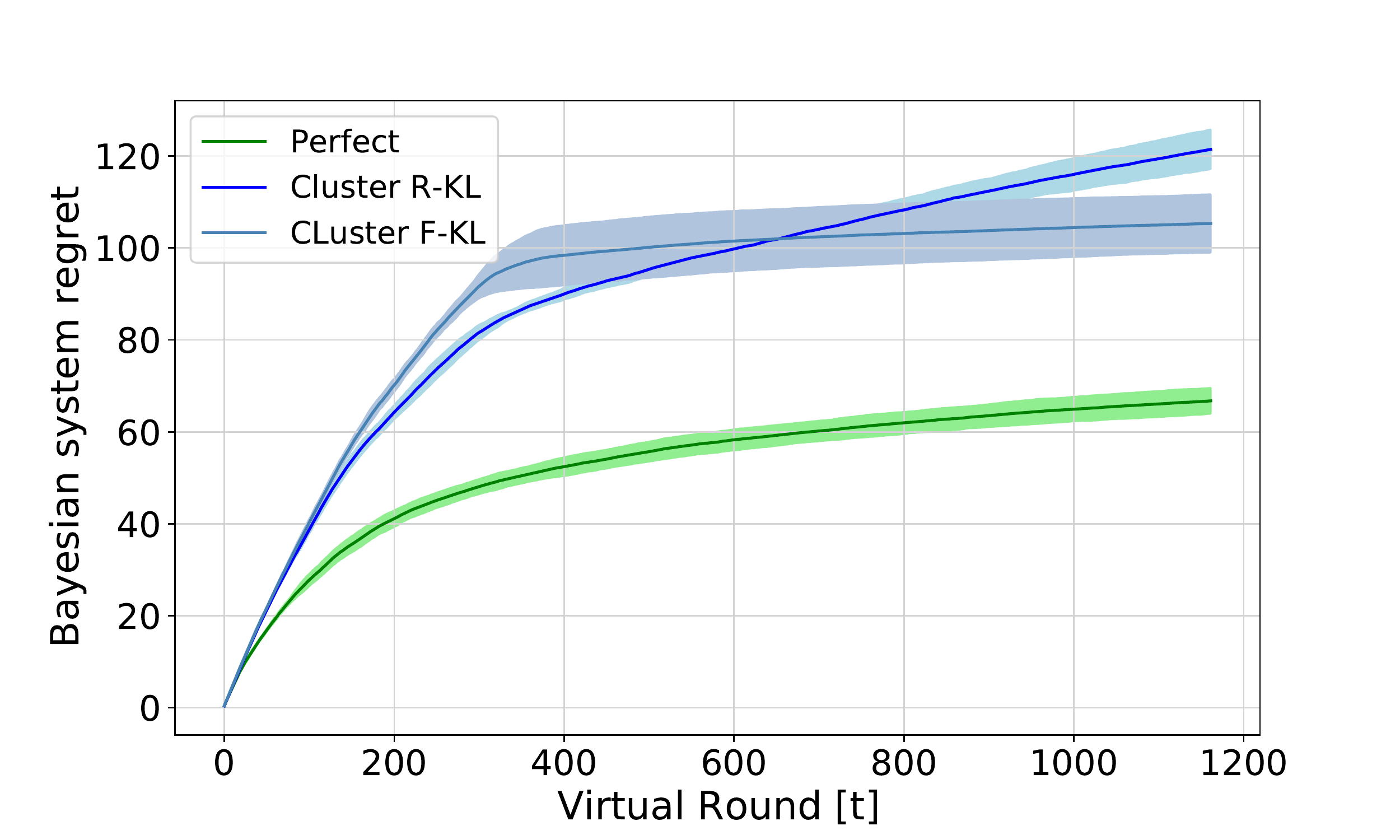}
		\caption{}
		\label{fig:regret_practical_scheme_1}
	\end{subfigure}
	\caption{Rate $R_{\pi}$ needed to reliably transmit the policy $\pi$ as a function of the system round $j$, together with the imposed rate $R$ (a); average state regret obtained by \textit{Perfect} and \textit{Cluster} algorithms, as a function of the virtual round $t$ (b). In this case, $G=1$. Curves indicates average rewards $\pm$ one standard deviation over $5$ runs.}
\end{figure}

\section{Conclusion}
\label{sec:conclusion}

In this paper, we have proposed and studied a novel rate-constrained remote RL problem, called \gls{rccmab}. We first proved the performance of the \gls{ts} strategy when no rate-constrained is imposed, and the information-theoretic limit for the rate needed to achieve sub-linear regret. We then studied schemes to compress the decision-maker's policy whenever the available rate is not sufficient to reliably transmit the intended actions to the controller. We considered $\alpha$-divergence as the distortion metric in the rate-distortion problem, and provided the shape of the policies that theoretically optimize the rate-distortion trade-off in close-form, in the cases $\alpha \rightarrow 0$, and $\alpha \rightarrow 1$, which lead to the reverse and forward KL divergences, respectively. We further proposed a practical compression scheme that relies on the idea of context clustering, and can be adopted to minimize the two analyzed divergence functions. The numerical results confirmed the limit on the achievable rate, and the performance gain that can be achieved when using proper compression strategies in the rate-constrained cases. We further connected and discussed with experiments the relation between our policy compression schemes and the \textit{information bottleneck} approach in \gls{rl}. 

Future steps include the adoption of more advanced practical compression algorithms that reduce the gap between the cluster and optimal schemes. Moreover, the extension of the problem in the context of more general remote \gls{rl} problems, where the next state depends on the current state and the action taken, is an interesting and challenging direction, as modern algorithms involve the parallelization of the training process exploiting a multitude of agents interacting with many replicas of the same environment.

\bibliography{biblio}
\bibliographystyle{IEEEtran}

\newpage

\appendix

\subsection{Regret Analysis}
\label{proof:ts_regret}

In this section we prove Theorem~\ref{thm:sys_regret}. As explained Sec.~\ref{subsec:asymptotic_bounds}, we align the interactions between agents and environment in time, and consider virtual rounds $t \in \{ 1, \dots, J\cdot N\} = \mathcal{T} $, and rewrite the Bayesian system regret as follows

\begin{align*}
\text{BR}(\pi^{TS}, J) &= \mathbb{E} \left[ \sum_{n \in \mathcal{N}}
\sum_{j=1}^J \mu(s_{j,n}, a^*(s_{j,n}^n)) - \mu\left(s_{j,n},  A_{j,n} \right)\right] \\
& = \mathbb{E} \left[
\sum_{t=1}^{NJ} \mu(s_{t}, a^*(s_{t})) - \mu\left(s_t, A_t \right)\right] \\
& = \mathbb{E} \left[ \sum_{s \in \mathcal{S}} \sum_{t_s \in \mathcal{T}_s} \mu(s_{t_s}, a^*(s_{t_s})) - \mu\left(s_{t_s}, A_{t_s} \right)\right] \\
\end{align*}
where $\mathcal{T}_s = \{ t \in \mathcal{T} : s_t = s\}$, i.e., all time-steps where context $s$ was sampled, and $| \mathcal{T}_s| = T_s$.  We now define the upper and lower bounds for the reward average $\mu(s, a)$, which hold with high probability, and are used to bound the average per-round regret
\begin{align}
\label{eq:bounds}
U_t(s, a) &= \hat{\mu}_{t-1}(s, a) + \sqrt{\frac{2 + 6 \log T_s}{\phi_{t-1}(s, a)}} \\
L_t(s, a) &= \hat{\mu}_{t-1}(s, a) - \sqrt{\frac{2 + 6 \log T_s}{\phi_{t-1}(s, a))}}
\end{align} 
where $\hat{\mu}_t(s, a)$ is the empirical average return of $a$ with context $s$ at time $t$, and $\phi_{t}(s, a)$ is related to the number of times arm $a$ has been pulled until time $t$ with context $s$, and will be better explained later. We now use \textit{Proposition 1} from \cite{Russo2014}, which allows us to write
\begin{align}
	\text{BR}(\pi^{TS}, T) = \mathbb{E} \left[ \sum_{t \in \mathcal{T}} U_{t}\left(s_t, A_{t}\right) - \mu \left(s_t, A_{t}\right)\right] + \mathbb{E} \left[ \sum_{t \in \mathcal{T}} \mu \left(s_t, a^*(s_t)\right) - U_t\left(s_t, a^*(s_t)\right)\right],
\end{align}
and apply \textit{Proposition 2} in \cite{Russo2014} obtaining, whenever $T > SKN$,
\begin{align}
& \text{BR}(\pi^{TS}, T) \leq \sum_{s \in \mathcal{S}} \sum_{t_s \in \mathcal{T}_s} \mathbb{E} \left[  U_{t_s}\left(s, A_{t_s}\right) - L_{t_s} \left(s, A_{t_s} \right) \right] + SKN.
\end{align}
 
We now consider the single context regret. For every $s \in \mathcal{S}$, we define $\mathcal{T}_s^a = \{ t_s \in \mathcal{T}_s : A_{t_s} = a\}$ and $|\mathcal{T}_s^a| = T_s^a$, bounding the quantity

\begin{align*}
\sum_{t_s \in \mathcal{T}_s}  U_{t_s}\left(s, A_{t_s}\right) - L_{t_s}\left( s, A_{t_s}\right) &= \sum_{a \in \mathcal{A}} \sum_{t_s^a \in  \mathcal{T}_s^a}  
U_{t_s^a} (s, a) - L_{t_s^a} (s, a) \\
& \stackrel{(a)}{\leq} \sum_{a \in \mathcal{A}} \left[ 1 + 2\sqrt{2 + 6 \log T_s} \sum_{t_s^a \in \mathcal{T}_s^a} (1 + \phi_{t_s^a-1}(s, a)) ^{-\frac{1}{2}} \right]
\end{align*}
Here (a) is a consequence of the bounds defined in Eq.~(\ref{eq:bounds}). In standard \gls{mab} analysis, as in \cite{Russo2014}, $\phi_t(s, a)$ is basically the number of times arm $a$ has been sampled up to time $t$ with context $s$, that we denote with $C_t(s, a)$. This quantity is critical to bound the intervals $U_t \left(s, a\right) - L_t\left( s, a \right)$. However, in our formulation, we define $\phi_t(s, a) = C_{j(t)}(s, a)$, where 
$$
j(t) = \floor{t/N} \cdot N
$$
is the last time the policy has been updated, i.e., at the end of the previous round.  Now, $\phi_t(s, a) \geq C_{t}(s, a) - N$, since it is not possible to sample more than $N$ times the same arm in one round, given that the number of agents is $N$. However, the quantity $ (1 + C_{t}(s, a) - N)^{-\frac{1}{2}}$ is not defined for $C_{t}(s, a) < N$. If $C_{t}(s, a) < N$, we can always write $ \phi_t(s, a) \geq \frac{C_{t}(s, a)}{N} + 1$.

\subsubsection{Finite-Time Analysis}

With the observations above, we can now prove an upper bound for the finite-time regret. We first notice that 
\begin{align*}
\sum_{t_s^a \in \mathcal{T}_s^a} (1 + \phi_{t_s^a-1}(s, a))^{-\frac{1}{2}} & \leq \sum_{t_s^a \in \mathcal{T}_s^a} \left(2 + \frac{C_{t_s^a-1}(s, a)}{N}\right) ^{-\frac{1}{2}} \\
&= \sum_{j=0}^{T_s^a-1} \left(2 + \frac{j}{N}\right) ^{-\frac{1}{2}} \\
&\leq\sum_{j=0}^{T_s^a-1} \left(\frac{j+1}{N}\right) ^{-\frac{1}{2}} \\
&=\sum_{j=1}^{T_s^a} \left(\frac{j}{N}\right) ^{-\frac{1}{2}}. 
\end{align*}
We can now use the integral bound to find
$$
\sum_{j=1}^{T_s^a} \left(\frac{j}{N}\right) ^{-\frac{1}{2}} \leq \sqrt{N}\int_{\tau=0}^{T_s^a} \tau ^{-\frac{1}{2}} d\tau = 2\sqrt{N}\sqrt{T_s^a}
$$
If we put all together, we obtain
\begin{align*}
\text{BR}(\pi, T) & \leq \sum_{s \in \mathcal{S}} \sum_{a \in \mathcal{A}} 1 + 4\sqrt{2+6\log T_s}\sqrt{N} \sqrt{T_s^a} \\
& \leq \sum_{s \in \mathcal{S}} K + 4\sqrt{2+6\log T_s}\sqrt{N} \sum_{a \in \mathcal{A}} \sqrt{T_s^a} \\
& \stackrel{(a)}{\leq} \sum_{s \in \mathcal{S}} K + 4\sqrt{(2+6\log T_s) N K\sum_{a \in \mathcal{A}} T_s^a} \\
& = KS + \sum_{s \in \mathcal{S}} 4 \sqrt{(2+6 \log T_s)K N T_s} \\
& \leq KS + 4 \sqrt{(2+6 \log T)K N} \sum_{s \in \mathcal{S}} \sqrt{T_s} \\
& \stackrel{(b)}{\leq} 2KS + 4 \sqrt{(2+6 \log T) K N S T}.
\end{align*}

The equalities (a) and (b) come from the Cauchy-Schwartz inequality, and the definitions of $T_s^a$, $T_s$, and $T$.  
%
%

\subsubsection{Asymptotic Regret Analysis}

However, in the asymptotic case $N \rightarrow \infty$, we get rid of the first constant terms when arms are pulled less than $N$ times. Consequently, we can use the bound $\phi_t(s, a) \geq C_{t}(s, a) - N$, and follow the same analysis for the finite-time case. We can see that the factor $T$ does not scale with $N$, obtaining 
\begin{align*}
\text{BR}(\pi, T) \in \mathcal{O}\left( \sqrt{KTS\log T}\right).
\end{align*}


\subsection{Achievable Rate}
\label{proof:rate}

In this section we provide the detailed proofs of Lemma~\ref{lemma:min_rate} and Lemma~\ref{lemma:achivable_rate}. To this end, we denote by $H_{q}(X)$ and $I_q(X;Y)$ the entropy and the mutual information computed with respect to the probability distribution $q$. Again, as introduced in Sec.~\ref{subsec:achievable_rate}, $H(A^*)$ is the entropy of the optimal actions, i.e., computed with respect to $\pi^*(a) = \sum_{s \in \mathcal{S}} P(s) \pi(a^*|s)$, where $\pi^*(a|s)$ is the optimal policy in context $s$.

We start by proving the following lemma.

\begin{lemma}[]
	\label{lemma:rate}
	Assuming that Thompson Sampling policy $\pi_j(a|s)$ achieves sub-linear Bayesian system regret, then $\lim_{j \rightarrow \infty} \text{I}_{\pi_j}(S;A) = \lim_{j \rightarrow \infty} H_{\pi_j}(A) = H(A^*)$.
\end{lemma}
\begin{proof}
	First of all we notice that, in order to achieve sub-linear Bayesian system regret, it is necessary to achieve sub-linear regret in all contexts $ s \in \mathcal{S}$, given the assumption that $\forall s \in \mathcal{S} ~ P_S(s) > 0$. We then write $I_{\pi_j}(S;A) = H_{\pi_j}(A) - H_{\pi_j}(A|S)$. Following Theorem 2 from \cite{ayfer2020}, if $\pi_j(a|s)$ achieves sub-linear regret, then $\forall s \in \mathcal{S}$
	\begin{align*}
	\lim_{j \rightarrow \infty} \pi_j(a = a^*|s) &= 1  \quad \text{when $a^*$ is the optimal arm}, \\
	\lim_{j \rightarrow \infty} \pi_j(a = a'|s) &= 0  \quad \text{when $a' \neq a^*$}.
	\end{align*}
	Consequently, in the limit, $\pi_j(a|s)$ is a deterministic function, thus
	\begin{align*}
	\lim_{t \rightarrow \infty} H_{\pi_t}(A^* | S) = 0,
	\end{align*}
	which concludes our proof.
\end{proof}
We start by proving Lemma~\ref{lemma:achivable_rate}, which we repeat below.

\textbf{Lemma~\ref{lemma:achivable_rate}}
If $R > H(A^*)$, then it is possible to achieve sub-linear Bayesian system regret, in the limit $N \rightarrow \infty$.
\begin{proof}

	We denote with $R_{\pi_j}$ the rate needed to perfectly convey the \gls{ts} policy to the controller at round $j$, and let $\epsilon > 0$ s.t. ${R = H(A^*) + \epsilon}$, where $R$ is the available communication rate. We now provide a scheme that guarantees sub-linear regret.
	
	If the policy $\pi_j$ generated from \gls{ts} has $R_{\pi_j} \leq H(A^*)$ $\forall j = 1, \dots, J$, then Theorem~\ref{thm:sys_regret} ensures sub-linear Bayesian system regret if sampling with $\pi_j$. If $\exists j$ such that $R_{\pi_j} > R$, generate parameters $\rho_j$ such that $\rho_j \in o(1)$ and $\sum_{j=1}^{\infty} \rho_j = \infty$, as explained in \cite{approx_ts}, Theorem 3. Then, with probability $\rho_t$ play $a$ uniformly at random, and with probability $1-\rho_t$, play according to a policy $Q_j(a|s)$, which satisfies the rate-distortion constraint $I_{Q(a|s)}(S;A) \leq R$, which can be transmitted to the controller, using as distortion measure the reverse KL divergence $d(Q_{SA}, \pi_{SA}) = \mathbb{E}_{P_S} \left[ D_{KL}(Q(\cdot || S) || \pi(\cdot || S))\right]$. Following Lemma 14 in \cite{russo16}, with this strategy enough exploration is guaranteed for the posterior policy, i.e., the one stored by the decision-maker, to concentrate. Consequently, by Lemma~\ref{lemma:rate}, there exists a finite $j_0$ s.t. $\forall j > j_0$, $R_{\pi_j} < H(A^*) + \epsilon$, in the limit $N \rightarrow \infty$. This means that, for the first $j_0$ rounds, both the target and the approximating policies would play sub-optimal arms with non-zero probabilities. Then, $\forall j > j_0$, it is possible to play the exact \gls{ts} policy, leading to the optimal decisions for all future steps, and hence, to a sub-linear regret. The above procedure holds $\forall \epsilon >0$, and so $\forall R > H(A^*)$.
\end{proof}

As we can see, the strategy that achieves sub-linear regret consists in using \gls{ts} at the decision-maker, which updates the posteriors according to the Bayes rule, and from which a target policy $\pi^*$ is computed. Then, the decision-maker computes an approximate policy $Q_j(a|s)$ as indicated in the proof, and it samples the arms according to $Q_j(a|s)$. We further notice that the limit $H(A^*)$ serves as a lower bound for practical schemes, as to achieve it we need $N \rightarrow \infty$.

We observe that Theorem 2 in \cite{ayfer2020} states that, if the \gls{ts} strategy achieves sub-linear regret, then the policy converges to the deterministic policy selecting with probability one the optimal action. However, the converse is not always true in general, i.e., there could exist policies that play sub-optimal arms infinitely many times as $j \rightarrow \infty$, and still achieve sub-linear regret. The point is that such policies must play the optimal arms for most of the rounds, and could pull sub-optimal arms for a sub-linear amount of rounds. However, to play optimally in one round, the decision-maker needs $ R \geq R_{\pi*} = H(A^*)$. We are now ready to prove Lemma~\ref{lemma:min_rate}.

\begin{lemma}[]
	\label{lemma:min_rate_ts}
If $R < H(A^*)$, it is not possible to achieve sub-linear Bayesian system regret.
\end{lemma}

\begin{proof}
As explained above, to play optimally in one round, the decision-maker needs $R \geq R_{\pi^*} = H(A^*)$, in the limit $N \rightarrow \infty$. If $R < H(A^*)$, as a consequence of Eq.~(\ref{eq:rate_dist}), the policy $Q(a|s)$ conveyed to the controller has non-zero distortion $d(Q_{SA}, \pi^*_{SA}) = D > 0$. If we take the $\text{L}_1$ norm as distortion measure, $Q(a|s)$ would sample a sub-optimal arm with constant probability of at least $D$ in every round $j = 1, \dots, J$. Consequently, sub-linear regret cannot be achieved.
\end{proof}

\subsection{Policy Compression}
\label{proof:policy_compression_derivation}

To compute the optimal policy that solves Eq.~(\ref{eq:rate_dist}) with a specific distortion function, we applied the well known Blahut-Arimoto iterative algorithm~\cite{cover:IT} that, given the considered distortion functions, is guaranteed to converge to the solution~\cite{blahu_convergence}.

We rewrite the optimization objective of Eq.~(\ref{eq:rate_dist}) as a double minimization problem (Sec. 10.8, \cite{cover:IT})
\begin{equation}
\label{eq:double_min}
R(D) = \min_{\tilde{Q}(a)} \min_{Q_{A|S} : d(Q_{SA}, P_{SA}) \leq D} \sum_{s,a} P_S(s) Q(a|s) \log_2 \frac{Q(a|s)}{\tilde{Q}(a)}.
\end{equation}
Following (Lemma 10.8.1, \cite{cover:IT}), the marginal $\tilde{Q}(y) = \sum_x P(x) Q(y|x)$ has the property 
\begin{equation}
\label{eq:outer_min}
\tilde{Q}(y) = \arg\min_{Q(y)} D_{KL} (P(x)Q(y|x) || P(x)Q(y)),
\end{equation}
that is, it minimizes the KL-divergence between the joint and the product $P(x)Q(y)$. This means that $\tilde{Q}(a)$ obtained by solving Eq.~(\ref{eq:double_min}) is indeed the marginal over the arms induced by $Q(a|s)$. Exploiting this formulation, it is possible to apply the iterative Blahut-Arimoto algorithm to solve the problem and find the solution \cite{cover:IT}. The process is initialized by setting a random $\tilde{Q}_0(a)$, which is used as a fixed point to compute
\begin{equation}
\label{eq:inner_min}
Q_1^*(a|s)= \text{argmin}_{Q_{A|S} : d(Q_{SA}, P_{SA}) \leq D} \sum_{s} P(s) \sum_a Q(a|s) \log_2 \frac{Q(a|s)}{\tilde{Q}_0(a)}.
\end{equation}
From $Q_1^*(a|s)$, we compute the optimal $\tilde{Q}_1(a)$ by solving Eq. (\ref{eq:outer_min}), which is simply the marginal $\tilde{Q}_1(a) = \sum_s P(s) Q_1^*(a|s)$. The process is iterated until convergence. We now solve the inner minimization problem, i.e., Eq.~(\ref{eq:inner_min}) with fixed $\tilde{Q}(a)$ and distortion $\E_{P_S} \left[ D_{\alpha}(Q, \pi) \right]$, for $\alpha \rightarrow 1$, and $\alpha \rightarrow 0$.

\subsubsection{Reverse KL Divergence ($\alpha \rightarrow 0$)}
\label{subsub:reverse_kl_proof}

To solve this problem, we solve the related Lagrangian

\begin{align*}
\mathcal{L}(Q(a|s), \lambda, \mu) = & \sum_s P(s) \sum_a Q(a|s) \log \frac{Q(a|s)}{\tilde{Q}_a} + \lambda \left( \sum_s P(s) \sum_a Q(a|s) \log \frac{Q(a|s)}{\pi(a|s)} - D \right) +  \\
& + \mu \left( \sum_s P(s) \sum_a Q(a|s) - 1\right)
\end{align*}
where the Lagrangian multiplier $\lambda$ has to be optimized to meet the constraints on the divergence, whereas $\mu$ ensures that the solution is a probability distribution, i.e., the elements sum to one. The positivity constraints on the terms are already satisfied by the fact that the solution has an exponential shape. We first take the derivative of the Lagrangian w.r.t. to the terms $Q(a|s)$ and set it to zero
\begin{align*}
\frac{\partial \mathcal{L}(Q(a|s), \lambda, \mu)}{\partial Q(a|s)} & = P(s) \log \frac{Q(a|s)}{\tilde{Q}(a)} + P(s) - \sum_{s'} P(s') Q(a|s') \frac{P(s)}{\tilde{Q}(a)} +\\
& + \lambda P(s) \left( \log \frac{Q(a|s)}{\pi(a|s)} + 1 \right) + \mu P(s) = 0
\end{align*}
finding 
\begin{align*}
\log \frac{Q(a|s)^{1 + \lambda}}{\tilde{Q}(a) \pi(a|s)^{\lambda}} & = - (\lambda + \mu) \\
Q(a|s)^{1+\lambda} & = e^{-(\mu +\lambda)} \tilde{Q}(a) \pi(a|s)^{\lambda} \\
Q(a|s) & = e^{\frac{-(\mu +\lambda)}{1+\lambda}} \tilde{Q}(a)^{\frac{1}{1+\lambda}} \pi(a|s)^{\frac{\lambda}{1+\lambda}}.
\end{align*}
We now define $\gamma := \frac{1}{1 + \lambda}$, $\gamma \in [0, 1]$, and obtain the distribution
\begin{align}
Q_{\gamma}(a|s) = \frac{\tilde{Q}(a)^{\gamma} \pi(a|s)^{1-\gamma}}{\sum_{a' \in \mathcal{A}} \tilde{Q}(a')^{\gamma} \pi(a'|s)^{1-\gamma}}, \quad \forall s \in \mathcal{S},  a \in \mathcal{A} .
\end{align}
By the convexity of KL-Divergence and its triangular inequality, we know the solution lies on the boundary of the constraints, i.e., when $\E_{P_S} \left[ D_{KL}(Q_{\gamma} || \pi) \right] = \delta$. 

\subsubsection{Forward KL Divergence ( $\alpha \rightarrow 1$)}
\label{subsub:kl_proof}

The derivative of the Lagrangian is (here the normalization factor is added in the end)

\begin{align*}
\frac{\partial \mathcal{L}(Q(a|s), \lambda)}{\partial Q(a|s)} & = P(s) \log \frac{Q(a|s)}{\tilde{Q}(a)} - \lambda P(s)  \frac{\pi(a|s)}{Q(a|s)}
\end{align*}
and setting it to zero leads to 
\begin{align*}
\frac{\partial \mathcal{L}\left(Q(a|s), \lambda\right)}{\partial Q(a|s)} & =0 \\
\log \frac{Q(a|s)}{\tilde{Q}(a)} - \lambda  \frac{\pi(a|s)}{Q(a|s)} &= 0 \\
\log \frac{\tilde{Q}(a)}{Q(a|s)} + \lambda  \frac{\pi(a|s)}{Q(a|s)} &= 0. \\
\end{align*}
We now we define $x := \frac{1}{Q(a|s)}, ~ \alpha := \lambda \pi(a|s), ~ \beta := 1$, and $\gamma := \log \tilde{Q}(a)$, obtaining
\begin{align*}
\alpha x + \beta \log x + \gamma &= 0 \\
\alpha x + \log \alpha x + \gamma - \log \alpha &= 0 \\
e^{\alpha x} \alpha x = \alpha e^{-\gamma} \\
x = \frac{1}{\alpha} \mathbf{W}_0 \left( \alpha e^{-\gamma}\right)
\end{align*}
where $\mathbf{W}_0(\cdot)$ is the Lambert function \cite{lambert}. We can now replace the introduced variables with the original terms and normalize, obtaining
\begin{align*}
	Q_{\lambda}(a|s) = \frac{\lambda \pi(a|s) \mathbf{W}_0 \left( \frac{\lambda \pi(a|s)}{\tilde{Q}(a)}\right)}{\sum_{a' \in \mathcal{A}}\lambda \pi(a'|s) \mathbf{W}_0 \left( \frac{\lambda \pi(a'|s)}{\tilde{Q}(a')}\right)},
\end{align*}
with $\lambda$ such that $\E_{P_S} \left[ D_{KL}(\pi || Q_{\lambda} ) \right] = \delta$.

\subsection{Clustering Compression Schemes}
\label{proof:cluster_compression}

\subsubsection{Reverse KL Divergence}
\label{subsub:rev_cluster_centroids}

Again, we compute the optimal centroids by solving the Lagrangian
\begin{align*}
\mathcal{L}(\mu^c_a, \lambda) = \sum_{s \in \mathcal{S}_c} P(s) \sum_{a \in \mathcal{A}} \mu_a^c \log \frac{\mu_a^c}{\pi(a|s)} + \lambda \left( \sum_{a \in \mathcal{A}} \mu_a^c -1 \right) 
\end{align*}
taking its derivative and solving the equality
\begin{align*}
\frac{\partial \mathcal{L}(\mu^c_a, \lambda)}{\partial \mu^c_a } &= \sum_{s \in \mathcal{S}_c} P(s) \left( \log \frac{\mu_a^c}{\pi(a|s)} + 1 \right) + \lambda = 0
\end{align*}
finding 
\begin{align*}
\log \mu_a^c A\left(\mathcal{S}_c\right) & = \sum_{s \in \mathcal{S}_c} P(s) \log \pi(a|s) + A\left(\mathcal{S}_c\right) + \lambda \\
\log \mu_a^c & = \sum_{s \in \mathcal{S}_c} \frac{P(s)}{A\left(\mathcal{S}_c\right)} \log \pi(a|s) + 1 + \frac{\lambda}{A\left(\mathcal{S}_c\right)} \\
\mu^c &= \frac{\prod_{s \in \mathcal{S}_c} \pi(a | s)^{\frac{P(s)}{A\left(\mathcal{S}_c\right)}}}{Z},
\end{align*}
where $Z$ is the normalizing factor, obtaining the shape expressed in Eq.~(\ref{eq:centroids_kl}).

\subsubsection{Forward KL Divergence}
\label{subsub:kl_cluster_centroids}

In this case, the Lagrangian is
\begin{align*}
\mathcal{L}(\mu^c_a, \lambda) = \sum_{s \in \mathcal{S}_c} P(s) \sum_{a \in \mathcal{A}} \pi(a|s) \log \frac{\pi(a|s)}{\mu_a^c} + \lambda \left( \sum_{a \in \mathcal{A}} \mu_a^c -1 \right).
\end{align*}
We take the derivative, and set it equal to zero
\begin{align*}
\frac{\partial \mathcal{L}(\mu^c_a, \lambda)}{\partial \mu^c_a } &= \sum_{s \in \mathcal{S}_c} P(s) \left(- \frac{\pi(a|s)}{\mu^c_a}\right) + \lambda = 0
\end{align*}
finding 
\begin{align*}
\mu^c_a = \frac{\sum_{s \in \mathcal{S}_c} P_S(s) \pi(a|s)}{Z}
\end{align*}
where $Z$ is the normalizing factor.

\end{document}

%% file: math_commands.tex
\usepackage{amsmath,amsfonts,bm}

\def\eqref#1{equation~\ref{#1}}

\def\floor#1{\lfloor #1 \rfloor}
\def\1{\bm{1}}

\DeclareMathAlphabet{\mathsfit}{\encodingdefault}{\sfdefault}{m}{sl}
\SetMathAlphabet{\mathsfit}{bold}{\encodingdefault}{\sfdefault}{bx}{n}

\newcommand{\E}{\mathbb{E}}

\DeclareMathOperator*{\argmin}{arg\,min}

%% file: acronyms.tex
\newacronym{iot}{IoT}{Internet of Things}
\newacronym{ml}{ML}{machine learning}
\newacronym{dl}{DL}{Deep Learning}
\newacronym{marl}{MARL}{multi-agent reinforcement learning}
\newacronym{rl}{RL}{reinforcement learning}
\newacronym{decpomdp}{Dec-POMDP}{Decentralized Partially Observable Markov Decision Process }
\newacronym{pomdp}{POMDP}{Partially Observable Markov Decision Process}
\newacronym{uav}{UAV}{Unmanned Aerial Vehicle}
\newacronym{dqn}{DQN}{Deep Q-Network}
\newacronym{dnn}{DNN}{deep neural network}
\newacronym{dial}{DIAL}{Differentiable Inter-Agent Learning}
\newacronym{mdp}{MDP}{Markov decision process}
\newacronym{fov}{FoV}{Field of View}
\newacronym{cnn}{CNN}{Convolutional Neural Network}
\newacronym{nn}{NN}{neural network}
\newacronym{ddql}{DDQL}{Distributed Deep Q-Learning}
\newacronym{pdf}{PDF}{Probability Density Function}
\newacronym{ndpomdp}{ND-POMDP}{Networked Distributed Partially Observable Markov Decision Process}
\newacronym{radam}{RAdam}{Rectified Adam}
\newacronym{cdf}{CDF}{cumulative distribution function}
\newacronym{mpc}{MPC}{Model Predictive Control}
\newacronym{rv}{rv}{Random Variable}
\newacronym{qoe}{QoE}{Quality of Experience}
\newacronym{tlc}{TLC}{Telecommunications}
\newacronym{cml}{CML}{communications for machine learning}
\newacronym{mlc}{MLC}{machine learning for communications}
\newacronym{drl}{DRL}{deep reinforcement rearning}
\newacronym{rf}{RF}{Radio Frequency}
\newacronym{urllc}{URLLC}{Ultra-Reliable and Low-Latency Communications}
\newacronym{fl}{FL}{federated learning}
\newacronym{kpi}{KPI}{Key Performance Indicators}
\newacronym{mec}{MEC}{Mobile Edge Computing}
\newacronym{ei}{EI}{Edge Intelligence}
\newacronym{bs}{BS}{base station}
\newacronym{sdn}{SDN}{Software Defined Networking}
\newacronym{mimo}{MIMO}{Multiple-Input Multiple-Output}
\newacronym{gp}{GP}{Gaussian Process}
\newacronym{iiot}{IIoT}{Industrial Internet of Things}
\newacronym{csi}{CSI}{Channel State Information}
\newacronym{sgd}{SGD}{Stochastic Gradient Descent}
\newacronym{iid}{i.i.d.}{independent and identically distributed}
\newacronym{ofdm}{OFDM}{Orthogonal Frequency Division Multiplexing}
\newacronym{los}{LOS}{Line-of-Sight}
\newacronym{nlos}{NLOS}{Non-Line-of-Sight}
\newacronym{snr}{SNR}{Signal to Noise Ratio}
\newacronym{rb}{RB}{Resource Block}
\newacronym{6g}{6G}{sixth generation}
\newacronym{ai}{AI}{artificial intelligence}
\newacronym{sfl}{SFL}{Synchronous Federated Learning}
\newacronym{frfl}{FRFL}{Fixed Rate Federated Learning}
\newacronym{pgm}{PGM}{Probabilistic Graphical Model}
\newacronym{hmm}{HMM}{Hidden Markov Model}
\newacronym{elbo}{ELBO}{Evidence Lower Bound}
\newacronym{pmf}{PMF}{Probability Mass Function}
\newacronym{smab}{SMAB}{Stochastic Multi-Armed Bandit}
\newacronym{mab}{MAB}{multi-armed bandit}
\newacronym{mc}{MC}{Monte Carlo}
\newacronym{is}{IS}{Importance Sampling}
\newacronym{dms}{DMS}{discrete memoryless source}
\newacronym{ucb}{UCB}{upper confidence bound}
\newacronym{ser}{SER}{Symbol Error Rate}
\newacronym{sc}{SC}{Semantic Communications}
\newacronym{voi}{VoI}{Value of Information}
\newacronym{nlp}{NLP}{natural language processing}
\newacronym{ts}{TS}{Thompson Sampling}
\newacronym{cmab}{CMAB}{contextual multi-armed bandit}
\newacronym{rccmab}{RC-CMAB}{rate-constrained \gls{cmab}}
\newacronym{rcmab}{R-CMAB}{remote \gls{cmab}}
\newacronym{ib}{IB}{information bottleneck}
\newacronym{merl}{MERL}{maximum entropy reinforcement learning}